\documentclass{article} % For LaTeX2e
\usepackage{nips15submit_e,times}

\usepackage{url}

\usepackage{times}
\usepackage{epsfig}
\usepackage{graphicx}
\usepackage{amsmath}
\usepackage{amssymb}
\usepackage{bbold}
\usepackage{dsfont}

\usepackage{enumerate}

\usepackage[toc,page]{appendix}
\usepackage[font=small,labelfont=bf]{caption}

\usepackage{xcolor}
\usepackage{amsthm}
\usepackage{amsfonts}
\usepackage{caption}
\usepackage{subcaption}
\usepackage{bm}
\usepackage{isomath}
\usepackage[numbers]{natbib}
\usepackage{footmisc}
\usepackage{stmaryrd}
\usepackage{fixltx2e}
\usepackage{dblfloatfix}
\usepackage{pbox}
\usepackage{capt-of}
\usepackage[normalem]{ulem}
\usepackage{multirow}

\usepackage{colortbl}

\makeatletter
\@namedef{ver@everyshi.sty}{}
\makeatother
\usepackage{pgf}

\newcommand{\xpt}{\edef\f@size{\@xpt}\rm}

\nipsfinaltrue

\def\ie{\emph{i.e.}}

\renewcommand\vec[1]{\ensuremath\boldsymbol{#1}}
\renewcommand\cdots{...}

\newcommand{\vy}{\mathbf{y}}

\newcommand{\tX}{\vec{\mathcal{X}}}

\newcommand{\mX}{\mathbf{X}}
\newcommand{\vx}{\mathbf{x}}

\newcommand{\mbr}[1]{\mathbb{R}^{#1}}

\newcommand{\vv}{\mathbf{v}}

\newcommand{\idx}[1]{\mathcal{I}_{#1}}
\newcommand{\semipd}[1]{\mathcal{S}_{+}^{#1}}
\newcommand{\spd}[1]{\mathcal{S}_{++}^{#1}}

\newcommand{\vzeta}{\boldsymbol{\zeta}}
\newcommand{\vc}{\mathbf{c}}

\newcommand{\vphi}{\boldsymbol{\phi}}

\newcommand{\bigoh}{\mathcal{O}}
\newcommand{\mPsi}{\vec{\Psi}}
\newcommand{\vj}{\vec{j}}

\newcommand{\enorm}[1]{\left\|{#1}\right\|_2}

\newcommand{\set}[1]{\left\{#1\right\}}

\DeclareMathOperator*{\sym}{Sym}

\DeclareMathOperator*{\trace}{Tr}
\DeclareMathOperator*{\kronstack}{\uparrow\!\otimes}

\DeclareMathOperator*{\avg}{avg}

\newcommand{\expl}[1]{\text{e}^{#1}}
\DeclareMathOperator*{\res}{Res}
\DeclareMathOperator*{\asinh}{Asinh}

\newcommand{\suptensorr}[2]{\mathfrak{S}^{#1}_{\times^{#2}}}

\newtheorem{proposition}{Proposition}
\newtheorem{remark}{Remark}

\newcommand{\mLambda}{\bm{\lambda}}
\newcommand{\mU}{\bm{U}}

\newcommand{\vphibar}{\boldsymbol{\bar{\phi}}}

\def\eg{\emph{e.g.}}

\newcommand{\mygthree}[1]{\boldsymbol{\mathcal{G}}\!\left(\!#1\!\right)}

\newcommand{\tG}{\boldsymbol{\mathcal{G}}}

\newcommand{\mIdent}{\boldsymbol{\mathds{I}}}
\newcommand{\sIdent}{\mathds{I}}
\newcommand{\vOnes}{\mathbb{1}}

\newcommand{\mJ}{\mathbf{J}}

\newcommand{\mC}{\mathbf{C}}

\newcommand{\tNnb}{\mathcal{N}}

\newcommand{\mPhi}{\boldsymbol{\Phi}}

\newcommand{\mM}{\boldsymbol{M}}

\newcommand{\mW}{\boldsymbol{W}}
\newcommand{\mD}{\boldsymbol{D}}
\newcommand{\mT}{\boldsymbol{T}}

\newcommand{\vmu}{\boldsymbol{\mu}}

\newcommand{\vvarphi}{\boldsymbol{\varphi}}

\newcommand{\stkout}[1]{{\ifmmode\text{\sout{\ensuremath{#1}}}\else\sout{#1}\fi}}

\DeclareMathOperator*{\arcsinh}{arcsinh}
\newcommand{\comment}[1]{}

\title{\Large A Deeper Look at Power Normalizations}

\author{Piotr Koniusz\thanks{Both authors contributed equally.\newline\indent\indent$\!\!$This work is published in CVPR'18. Please respect the authors' efforts by not copying/plagiarizing bits and pieces of this work for your own gain. If you find anything inspiring in this work, be kind enough to cite it.}\textsuperscript{$\;\,$,1,2}\qquad Hongguang Zhang\textsuperscript{$*$,2,1}\qquad Fatih Porikli\textsuperscript{2}\\
$^1$Data61/CSIRO, $^2$Australian National University\\
firstname.lastname@\{data61.csiro.au\textsuperscript{1}, anu.edu.au\textsuperscript{2}\}
}

\newcommand\keywords[1]{}
\nipsfinalcopy % Uncomment for camera-ready version

\begin{document}

\maketitle

\def\arxiv{arxiv}
\begin{abstract}
Power Normalizations (PN) are very useful non-linear operators in the context of Bag-of-Words data representations as they tackle problems such as feature imbalance. In this paper, we reconsider these operators in the deep learning setup by introducing a novel layer that implements PN for non-linear pooling of feature maps. Specifically, by using a kernel formulation, our layer combines the feature vectors and their respective spatial locations in the feature maps produced by the last convolutional layer of CNN. Linearization of such a kernel results in a positive definite matrix capturing the second-order statistics of the feature vectors, to which PN operators are applied. We study two types of PN functions, namely (i) MaxExp and (ii) Gamma, addressing their role and meaning in the context of non-linear pooling. We also provide a probabilistic interpretation of these operators and derive their surrogates with well-behaved gradients for end-to-end CNN learning. We apply our theory to practice by implementing the PN layer on a ResNet-50 model and showcase experiments on four benchmarks for fine-grained recognition, scene recognition, and material classification. Our results demonstrate state-of-the-art performance across all these tasks.
\end{abstract}

\ifdefined\arxiv
\else
\begin{figure*}[t]%htbp % left bottom right top
\vspace{-0.15cm}
\centering%%%%\vspace{-0.3cm}
%
%%\hspace{0.1cm}
%\begin{subfigure}[b]{0.99\linewidth}
\centering\includegraphics[trim=0 0 0 0, clip=true, width=14.4cm]{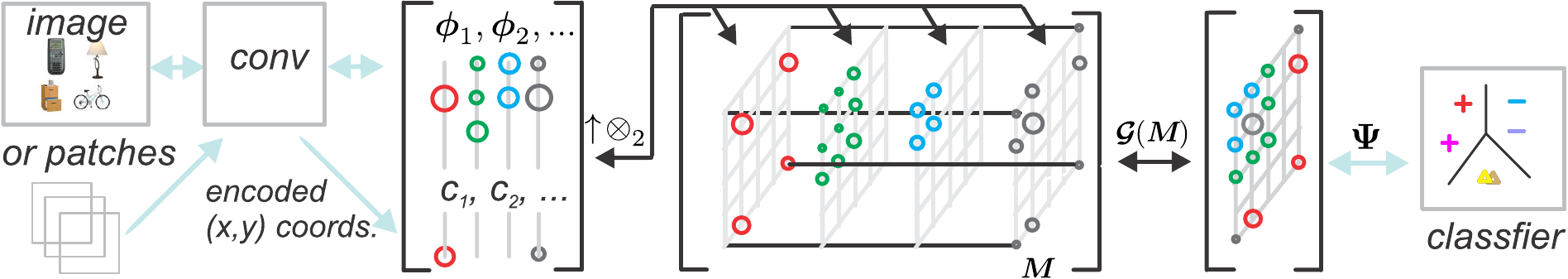}
%%\phantomcaption
%\caption{\label{fig:cnn1}}
%\vspace{-0.1cm}
%\end{subfigure}
%
%\begin{subfigure}[b]{0.99\linewidth}
%\centering\includegraphics[trim=0 0 0 0, clip=true, width=8.2cm]{images/test.pdf}
%%\phantomcaption
%\caption{\label{fig:cnn2}}
%\end{subfigure}
%
\vspace{-0.2cm}
\caption{Our end-to-end pipeline. We pass an image (or patches) to CNN and extract feature vectors $\vphi$ from its last conv. layer and augment them by encoded spatial coordinates $\vc$. We perform pooling on second-order matrix $\mM$ by the Power Normalization function $\tG$.}\vspace{-0.3cm}
\label{fig:principle}
\end{figure*}
\fi

\section{Introduction}
\label{sec:intro}
Second-order statistics of data features have played a pivotal role in advancing the state of the art on several problems in computer vision, including object recognition, texture categorization, action representation, and human tracking, to name a few of applications \cite{tuzel_rc,porikli2006tracker,wang2011tracking,locallog-euclidean,guo2013action,carreira_secord,me_tensor}. For example, in the popular region covariance descriptors~\cite{tuzel_rc}, a covariance matrix, which is computed over multi-modal features from image regions, is used as an object representation for recognition and tracking, and has been extended to several other applications~\cite{tuzel_rc, porikli2006tracker,wang2011tracking,locallog-euclidean,guo2013action}. Given Bag-of-Words histograms or local descriptor vectors from an image, a second-order co-occurrence pooling of these vectors captures the occurrences of two features together. Such a strategy has been recently shown to result in a superior performance in semantic segmentation and visual concept detection, compared to their first-order counterparts~\cite{carreira_secord,me_tensor_tech_rep,me_tensor}. A natural extension led to higher-order pooling operators~\cite{me_tensor_tech_rep,me_tensor,sparse_tensor_cvpr} on third-order super-symmetric tensors which improve results over the second-order descriptors over 7\% MAP on PASCAL VOC07.

%\vspace{-0.05cm}
However, second and higher-order statistics require appropriate aggregation and pooling mechanisms to obtain the highest classification results \cite{carreira_secord,me_tensor_tech_rep,me_tensor}. Once the statistics are captured in the matrix form, they undergo next a non-linearity such as Power Normalization \cite{me_ATN} which role is to reduce/boost contributions from frequent/infrequent visual stimuli in an image, respectively. A significant progress made by the Bag-of-Words model provides numerous insights into the role played by pooling during the aggregation step. The theoretical relation between Average and Max-pooling was studied in~\cite{boureau_midlevel}. A detailed likelihood-based analysis of feature pooling was conducted in~\cite{boureau_pooling} which led to a \emph{theoretical expectation of Max-pooling}, improving overall classification results. Power Normalization has also been applied to Average pooling by Fisher Kernels~\cite{perronnin_fisherimpr}. Max-pooling has been recognized as a lower bound of the likelihood of \emph{`at least one particular visual word being present in an image'}~\cite{liu_sadefense}. According to an evaluation \cite{me_ATN}, these pooling methods are all closely related. However, evaluations \cite{me_ATN} do not consider the second-order pooling scenario or end-to-end learning. In the context of second-order pooling, element-wise and eigenvalue Power Normalization (ePN) were both first proposed in \cite{me_tensor_tech_rep} in 2013.

%\vspace{-0.05cm}
In this paper, we aim to revisit the above pooling methods 
%In this paper, we revisit pooling 
in end-to-end setting and shed further light on their interpretation in the context of second-order matrices. Firstly, we propose a kernel formulation which combines feature vectors collected from the last convolutional layer of ResNet-50 together with so-called spatial location vectors,
previously explored in \cite{me_SCC,me_ATN,me_tensor_tech_rep} around 2011--2013, %and they were explored in \cite{me_SCC,me_ATN,me_tensor_tech_rep} around 2011--2013.
which contain spatial locations corresponding to feature vectors in the CNN feature maps. A linearization of such a kernel results in a second-order matrix which contains aggregated second-order statistics of these combined vectors. Subsequently, we focus on the role of the Power Normalization family in end-to-end setting. We show that these functions have a well-founded probabilistic interpretation in the context of second-order statistics. Moreover, we propose PN surrogates which have well-behaved derivatives suitable for back-propagation unlike typical PN functions.

Our contributions are three-fold: (i) we propose to aggregate feature vectors extracted from CNNs and their spatial coordinates into a second-order matrix by principled derivations in end-to-end manner, %(ii) we revisit Power Normalization family of functions and give their interpretation in the context of second-order statistics, 
(ii) we revisit Power Normalization functions, derive them for second-order representations and show that they follow Binomial or Multinomial distributions if features are drawn from the Brenoulli distribution, 
(iii) we propose PN surrogates with well-behaved derivatives for end-to-end learning, (iv) we propose new spectral variants of pooling. Figure \ref{fig:principle} shows our pipeline.

We perform evaluations on ResNet-50 and four image classification benchmarks such as Flower102, MIT67, FMD and Food101 where we demonstrate state-of-the-art results.

\ifdefined\arxiv
\begin{figure*}[t]%htbp % left bottom right top
\centering%%%%\vspace{-0.3cm}
%
%%\hspace{0.1cm}
%\begin{subfigure}[b]{0.99\linewidth}
\centering\includegraphics[trim=0 0 0 0, clip=true, width=14.0cm]{images/principle.pdf}
%%\phantomcaption
%\caption{\label{fig:cnn1}}
%\vspace{-0.1cm}
%\end{subfigure}
%
%\begin{subfigure}[b]{0.99\linewidth}
%\centering\includegraphics[trim=0 0 0 0, clip=true, width=8.2cm]{images/test.pdf}
%%\phantomcaption
%\caption{\label{fig:cnn2}}
%\end{subfigure}
%
\vspace{-0.2cm}
\caption{Our end-to-end pipeline. We pass an image (or patches) to CNN and extract feature vectors $\vphi$ from its last conv. layer and augment them by encoded spatial coordinates $\vc$. We perform pooling on second-order matrix $\mM$ by the Power Normalization function $\tG$.}\vspace{-0.3cm}
\label{fig:principle}
\end{figure*}
\fi

\section{Related Work}
\label{sec:related_work}

Second-order statistics have been extensively studied in the context of texture recognition \cite{tuzel_rc,tuzel2008detection, elbcm_brod} by the use of so-called Region Covariance Descriptors (RCD).

\vspace{0.05cm}
\noindent{\textbf{Region Covariance Descriptors (RCD).}} 
Such methods use a representation which typically captures co-occurrences of luminance, first- and/or second-order derivatives of texture patterns. Alternatively, co-occurrences in Local Binary Patterns (LBP) are captured to build second-order matrices \cite{elbcm_brod}. RCD approaches have also been successfully applied to tracking \cite{porikli2006tracker}, semantic segmentation \cite{carreira_secord} and object category recognition \cite{me_tensor}, to name but a few of applications. 
The design of RCD typically requires a decision on what signals need to be aggregated into the second-order representation and how to compare positive (semi-)definite datapoints resulting from such an aggregation step. There exist several non-Euclidean distances often applied to positive definite matrices which we list next.

\vspace{0.05cm}
\noindent{\textbf{Non-Euclidean distances.}} 
The distance between two positive definite datapoints is typically measured according to the Riemannian geometry while Power-Euclidean distances \cite{dryden_powereuclid} extend to positive semi-definite distances. 
In particular, Affine-Invariant Riemannian Metric \cite{PEN06,bhatia_pdm}, KL-Divergence Metric ({\em KLDM}) \cite{wang_jeffreys}, Jensen-Bregman LogDet Divergence ({\em JBLD}) \cite{anoop_logdet} and Log-Euclidean ({\em LogE}) \cite{arsigny2006log} have been used in the context of diffusion imaging and the RCD-based methods. % listed above. 
%
%Moreover, dictionary learning methods frequently employ Non-Euclidean distances. In \cite{mehrtash_dict_manifold,mehrtash_kernels2}, the authors propose a kernelized variant of sparse and locality-constrained coding approaches on the Riemannian manifold. As a consequence, they can learn dictionaries as well as kernel approximations. In \cite{beyond_gauss}, a supervised dimensionality reduction with use of f-divergences is proposed which are strongly related to geodesic distances.
Dictionary and metric learning methods also use non-Euclidean distances \cite{mehrtash_dict_manifold,mehrtash_kernels2,beyond_gauss,Roy_CVPR_2018,Kumar_CVPR_2018}. 

Our approach differs %from the above works 
in that we perform end-to-end learning in the CNN setting while RCD and dictionary learning constitute shallow architectures that perform worse than CNNs on the majority of classification tasks. 

We note that the Log-Euclidean distance and Power Normalization have been implemented in the CNN setting \cite{sminchisescu_matrix,vangol_riem_net,secord_peihua_li,lin2017improved} for the purpose of region classification. % and action recognition.
 These methods employ back-propagation %on the Log-Euclidean distance 
which requires costly eigenvalue decomposition for computations of derivatives deeming them computationally inefficient. Note that the cost of a single eigenvalue decomposition is at least $\bigoh(d^\omega)$, where constant $2\!<\!\omega\!<\!2.376${\color{red}\footnotemark[1]}. The typical bottleneck in using non-Euclidean distances in end-to-end setting lies in their costly back-propagation rules.

\footnotetext[1]{\label{foot:complexity1}We assume that the eigenvalue decomposition of large matrices ($d\!=\!4096$) in CUDA BLAS is fast and efficient--which is not the case.}

Our work differs in that we make an i.i.d. assumption on our co-occurrence features in our second-order representation. Thus, we require only element-wise rather than spectral operations.  %\eg, we assume that the underlying Gaussian captured by the second-order matrix is decorrelated--this means the eigenvalues of such a materix are non-negative and eigenvectors simply point in the direction of the axes of a Cartesian coordinate system as proposed in \cite{me_tensor}.
This reduces the complexity and relies on trivial arithmetic operations easy to implement on GPU.

\vspace{0.05cm}
\noindent{\textbf{Pooling and CNNs.}} 
There exist several approaches for image retrieval and recognition which perform some form of aggregation over first-order statistics extracted from the CNN maps \eg, \cite{orderless_pooling,deep_aggreg,netvlad}. In \cite{orderless_pooling}, the authors propose to extract multiple regions from an image and aggregate CNN responses into an image representations. In \cite{deep_aggreg}, the authors aggregate local deep features for the task of image retrieval. In \cite{netvlad}, the authors extend Vector of Locally Aggregated Descriptors (VLAD) to an end-to-end trainable system.

Our approach differs in that we use co-occurrences in end-to-end setting and take an analytical look at how to interpret Power Normalization functions in this setting.

There has been also a revived interest in creating co-occurrence patterns in CNN setting similar in spirit to RCD. Approach \cite{bilinear_finegrained} applies a fusion of two CNN streams via outer product in the context of the fine-grained image recognition. Another approach for face recognition \cite{face_cooc} uses co-occurrences of CNN feature vectors and facial attribute vectors to obtain state-of-the-art face recognition results. A recent approach \cite{deep_cooc} extracts feature vectors at two separate locations in feature maps and performs an outer product to form a CNN co-occurrence layer.

%Our paper differs in that we formally show how to derive co-occurrence matrices which capture second-order statistics. 
In contrast to these papers, we use symmetric positive (semi-)definite matrices rather than negative definite ones.%--the choice of geometry is an orthogonal direction to our investigations. We propose well-motivated pooling functions which are interpretable and improve classification results.

\vspace{0.05cm}
\noindent{\textbf{Power Normalizations.}} 
Practical image representations have to deal with the so-called burstiness which is `{\em the property that a given visual element appears more times in an image than a statistically independent model would predict}' \cite{jegou_bursts}. Power Normalization~\cite{boughorbel_intersect, perronnin_fisherimpr, jegou_bursts} is known to suppress this burstiness and has been extensively studied and evaluated in the context of Bag-of-Words \cite{me_ATN,me_tensor}. 
The theoretical relation between Average and Max-pooling was studied in~\cite{boureau_midlevel} which highlighted the underlying statistical reasons for the superior performance of Max-pooling compared to a mere average of feature vectors. An analysis of feature pooling was conducted in~\cite{boureau_pooling} under specific assumptions on distributions from which the aggregated features are drawn.  A relationship between the likelihood of `\emph{at least one particular visual word being present in an image}' and Max-pooling was studied in \cite{liu_sadefense}. According to a survey \cite{me_ATN}, these Power Normalization functions are closely related.

We take a similar view on PN functions, however, we devise an end-to-end trainable CNN layer and derive new pooling functions with well-behaved derivatives. We follow theoretical foundations of the Power Normalization family.

\section{Background}
\label{sec:background}

Below we review our notations and the background on kernel linearizations and the Power Normalization family. % of non-linear functions.

\subsection{Notations}
\label{sec:notations}
%Let $\vx\in\mbr{d}$ be a $d$-dimensional feature vector. $\idx{N}$ stands for the index set $\set{1, 2,\cdots,N}$. 
%
Let $\vx\in\mbr{d}$ be a $d$-dimensional feature vector. Then we use $\tX\!=\!{\kronstack}_r\,\vx$ to denote the $r$-mode super-symmetric rank-one tensor $\tX$ generated by the $r$-th order outer-product of $\vx$, where the element of $\tX\!\in\!\suptensorr{d}{r}$ at the $\left(i_1,i_2,\cdots, i_{r}\right)$-th index is given by $\Pi_{j=1}^r x_{i_j}$. $\idx{N}$ stands for the index set $\set{1, 2,\cdots,N}$. 
%
%
%The Frobenius norm of matrix is given by  $\fnorm{\mX}\!\!=\!\!\!\sqrt{\sum\limits_{m,n} \!\!X_{mn}^2}$, where $X_{mn}$ represents the $\left(m,n\right)$-th element of $\mX$.  
The spaces of symmetric positive semidefinite and definite matrices are $\semipd{d}$ and $\spd{d}$. Moreover, $\sym(\mX)\!=\!\frac{1}{2}(\mX\!+\!\mX^T\!)$. 
A vector with all coefficients equal one is denoted by $\vOnes$, $\vj_m$ is a vector of all zeros except for the $m$-th coefficient which is equal one, and $\mJ_{mn}$ is a matrix of all zeros with a value of one at the position $(m,n)$. 
%The MATLAB-style operators for matrix vectorization and matrix reshaping to the size of $(m,n)$ are denoted by $(:)$, \ie, $\mX_{(:)}$, and $\res(\mX,m,n)$.
Moreover, $\odot$ is the Hadamard product (element-wise multiplication). We use the MATLAB notation $\vv\!=\![\text{begin}\!:\!\text{step}\!:\!\text{end}]$ to generate a  vector $\vv$ with elements starting as {\em begin}, ending as {\em end}, with stepping equal {\em step}. Operator `$;$' in $[\vx; \vy]$ denotes the concatenation of vectors $\vx$ and $\vy$ (or scalars).

%
%Operator $\left[k_{ij}\right]_{i,j\!\in\!\idx{N}}$ denotes stacking coefficients $k_{ij}$ into matrix $\mK$ of size $N\!\times\!N$. 
%Moreover, $\delta(x)\!=\!\lim_{\sigma\rightarrow 0}\exp(-x^2/(2\sigma^2))$ returns one for $x\!=\!0$ and zero for $x\!\neq\!0$. We also define $\vOnes\!=\![1,\cdots,1]^T$.

\subsection{Kernel Linearization}
\label{sec:kernel_linearization}
In the sequel, we will use kernel feature maps %\eqref{eq:gauss_lin2} 
detailed below to embed  $(x,y)$  locations of feature vectors extracted from conv. CNN maps at $(x,y)$  into a non-linear Hilbert space. Such locations are called {\em spatial coordinates} \cite{me_SCC,me_tensor}.$\!\!$ %and they were explored in \cite{me_SCC,me_ATN,me_tensor_tech_rep} around 2011--2013.

\begin{proposition}
\label{pr:gaus_lin}
Let $G_{\sigma}(\vx\!-\!\vy)=\exp(-\!\enorm{\vx\!-\!\vy}^2/{2\sigma^2})$ denote a Gaussian RBF kernel centered at $\vy$ and having a bandwidth $\sigma$. Kernel linearization refers to rewriting $G_{\sigma}$ as an inner-product of two (in)finite-dimensional feature maps which we obtain via probability product kernels \cite{jebara_prodkers}.
Specifically, we employ the inner product of $d'$-dimensional isotropic Gaussians given $\vx,\vy\!\in\!\mbr{d'}\!$ as follows: %The resulting approximation can be written as:
\begin{align}
&\!\!\!\!\!\!\!G_{\sigma}\!\left(\vx\!-\!\vy\right)\!\!=\!\!\left(\frac{2}{\pi\sigma^2}\right)^{\!\!\frac{d'}{2}}\!\!\!\!\!\!\int\limits_{\vzeta\in\mbr{d'}}\!\!\!\!G_{\sigma/\sqrt{2}}\!\!\left(\vx\!-\!\vzeta\right)G_{\sigma/\sqrt{2}}(\vy\!\!-\!\vzeta)\,\mathrm{d}\vzeta.
\label{eq:gauss_lin}
\end{align}%\\[-15pt]
Eq. \eqref{eq:gauss_lin} can be approximated by replacing the integral with the sum over $Z$ pivots $\vzeta_1,\cdots,\vzeta_Z$. Thus, we obtain: % a feature map $\vvarphi$:
\begin{align}
&\!\!\!\!\!\!\!\vvarphi(\vx)=\left[{G}_{\sigma/\sqrt{2}}(\vx-\vzeta_1),\cdots,{G}_{\sigma/\sqrt{2}}(\vx-\vzeta_Z)\right]^T,\!\!\label{eq:gauss_lin2}\\
& \text{ and } G_{\sigma}(\vx\!-\!\vy)\approx\left<\sqrt{c}\vvarphi(\vx), \sqrt{c}\vvarphi(\vy)\right>,
\label{eq:gauss_lin3}
\end{align}
where $c$ is a constant. We refer to \eqref{eq:gauss_lin2} as a (kernel) feature map{\color{red}\footnotemark[3]} and to \eqref{eq:gauss_lin3} as the linearization of the RBF kernel. %, respectively.
\end{proposition}
\begin{proof}
The Gaussian kernel can be rewritten as a probability product kernel. See \cite{jebara_prodkers} (Section 3.1) for derivations.$\!\!\!$
\end{proof}

\subsection{Second- and Higher-order Tensors}
\label{sec:som}
Below we show that second- or higher-order tensors emerge from a linearization of sum of Polynomial kernels.
\begin{proposition}
\label{pr:linearize}
Let $\mPhi_A\equiv\{\vphi_n\}_{n\in\tNnb_{\!A}}$, $\mPhi_B\!\equiv\{\vphi^*_n\}_{n\in\tNnb_{\!B}}$ be datapoints from two images $\Pi_A$ and $\Pi_B$, and $N\!=\!|\tNnb_{\!A}|$ and $N^*\!\!=\!|\tNnb_{\!B}|$ be the numbers of data vectors \eg, obtained from the last convolutional feature map of CNN for images $\Pi_A$ and $\Pi_B$. Tensor feature maps result from a linearization of the sum of Polynomial kernels of degree $r$:
\begin{align}
& K(\mPhi_A, \mPhi_B)\!=\!\left<\mPsi(\mPhi_A),\mPsi(\mPhi_B)\right>\!=\label{eq:hok1}\\
& \!\!\!\!\frac{1}{NN^*\!}\!\!\sum\limits_{
n\in \tNnb_{\!A}}\sum\limits_{n'\!\in \tNnb_{\!B}\!}\!\!\!\left<\vphi_n, \vphi^*_{n'}\right>^r\!\text{ where } 
\mPsi(\mPhi)\!=\!\frac{1}{N}\sum\limits_{
n\in \tNnb}{\kronstack}_r\,\vphi_n.\!\!\nonumber
\end{align}
\end{proposition}
\begin{proof}
See \cite{me_domain} for the details of such an expansion.
\end{proof}

\begin{remark}
In what follows, we will use second-order matrices obtained from the above expansion for $r\!=\!2$, that is:
\begin{align}
& \!\!\!\!\!\!\!\frac{1}{NN^*\!}\!\!\sum\limits_{
n\in \tNnb_{\!A}}\sum\limits_{n'\!\in \tNnb_{\!B}}\!\!\!\left<\vphi_n, \vphi^*_{n'}\right>^2\!\!=\!
\Big\langle\frac{1}{N}\sum\limits_{
n\in \tNnb_{\!A}}{\vphi_n\vphi_n^T}, \frac{1}{N^*\!}\sum\limits_{
n\in \tNnb_{\!B}}{\vphi^*_{n'}{\vphi^*_{n'}}^{\!\!\!T}}\Big\rangle.\!\!\label{eq:hok2}
\end{align}
\footnotetext[3]{\label{foot:maps}Note that (kernel) feature maps are not conv. CNN maps. They are two separate notions that happen to share the same name.}
Thus, we obtain the following (kernel) feature map{\color{red}\footnotemark[3]}:
\begin{align}
& \mPsi\left(\{\vphi_n\}_{n\in\tNnb}\right)=\tG\Big(\frac{1}{N}\sum_{n\in\mathcal{N}}\!\vphi_n\vphi_n^T\Big),\label{eq:hok3}
\end{align}
where $\mygthree{\,\mX}\!=\!\mX$ will be later replaced by various Power Normalization functions.
\end{remark}

\subsection{Power Normalization Family}
\label{sec:pn}
%Likelihood based pooling methods have recently shed new light on the role of the pooling step in Bag-of-Words. 

Max-pooling \cite{boureau_midlevel} can be derived by drawing features from the Bernoulli distribution under the i.i.d. assumption \cite{boureau_pooling} which leads to so-called {\em Theoretical Expectation of Max-pooling} ({\em MaxExp}) operator \cite{me_ATN} detailed below.

\begin{proposition}\label{prop:maxexp}
Assume a vector $\vphi\!\in\!\{0,1\}^{N}$ which stores $N$ outcomes of drawing from Bernoulli distribution under the i.i.d. assumption for which the probability $p$ of an event $(\phi_{n}\!=\!1)$ and $1\!-\!p$ for $(\phi_{n}\!=\!0)$ can be estimated as an expected value \eg, $p\!=\!\avg_n\phi_n$. Then the probability of at least one positive event in $\vphi$ from $N$ trials becomes:
\vspace{-0.2cm}
\begin{equation}
\label{eq:my_maxexp1}
\psi\!=\!1\!-\!(1\!-\!p)^{N}.
%\vspace{-0.1cm}
\end{equation}
\end{proposition}
\begin{proof}
\label{pr:maxexp}
The proof follows the school syllabus for a fair coin toss. The probability of all $N$ outcomes to be $\{(\phi_{1}\!=\!0),\cdots,(\phi_{N}\!=\!0)\}$ amounts to $(1\!-\!p)^N$. 
The probability of at least one positive outcome 
$(\phi_{n}\!=\!1)$ amounts to applying the logical `or' $\{(\phi_{1}\!=\!1)\,|\cdots|\,(\phi_{N}\!=\!1)\}$ and leads to:
\vspace{-0.2cm}
\begin{equation}\label{eq:my_maxexp2}
1\!-\!(1\!-\!p)^{N}=\,\sum_{n=1}^{N} \binom{N}{n} p^n(1\!-\!p)^{N-n}.
\vspace{-0.5cm}
\end{equation}
\end{proof}
\begin{remark}
\vspace{-0.3cm}
\label{re:maxexp}
A practical implementation of this pooling strategy \cite{me_ATN} is given by $\psi_k\!=\!1\!-\!(1\!-\!\avg_n\phi_{kn})^{\eta}$, where $0\!<\!\eta\!\approx\!N$ is an adjustable parameter and $\phi_{kn}$ is a $k$-th feature of an $n$-th feature vector \eg, as defined in Prop. \ref{pr:linearize}, which is normalized to range 0--1.
\end{remark}
\begin{remark}
\label{re:pn}
It was shown in \cite{me_ATN} that Power Normalization ({Gamma}) given by $\psi_k\!=\!(\avg_n\phi_{kn})^\gamma$, where $0\!<\!\gamma\!\leq\!1$ is an adjustable parameter, is in fact an approximation of MaxExp.
\end{remark}
\section{Problem Formulation}
\label{sec:problem}

We start by devising our co-occurrence and pooling layers. % and their well-behaved approximations.
We show that the Power Normalization ({\em Gamma}) has an ill-behaved derivative. Thus, we generalize MaxExp and Gamma \cite{me_ATN,me_tensor} to Logistic a.k.a. Sigmoid ({\em SigmE}) and the Arcsin hyperbolic ({\em AsinhE}) functions.

\ifdefined\arxiv
\newcommand{\PowH}{3.0cm}
\newcommand{\PowHB}{2.875cm}
\newcommand{\PowW}{3.65cm}
\else
\newcommand{\PowH}{3.6cm}
\newcommand{\PowHB}{3.4cm}
\newcommand{\PowW}{4.5cm}
\fi

\begin{figure*}[t]%htbp % left bottom right top
\centering
%\vspace{-0.2cm}
\hspace{-0.3cm}
\begin{subfigure}[t]{0.24\linewidth}
\centering\includegraphics[trim=0 0 0 0, clip=true, height=\PowH, width=\PowW]{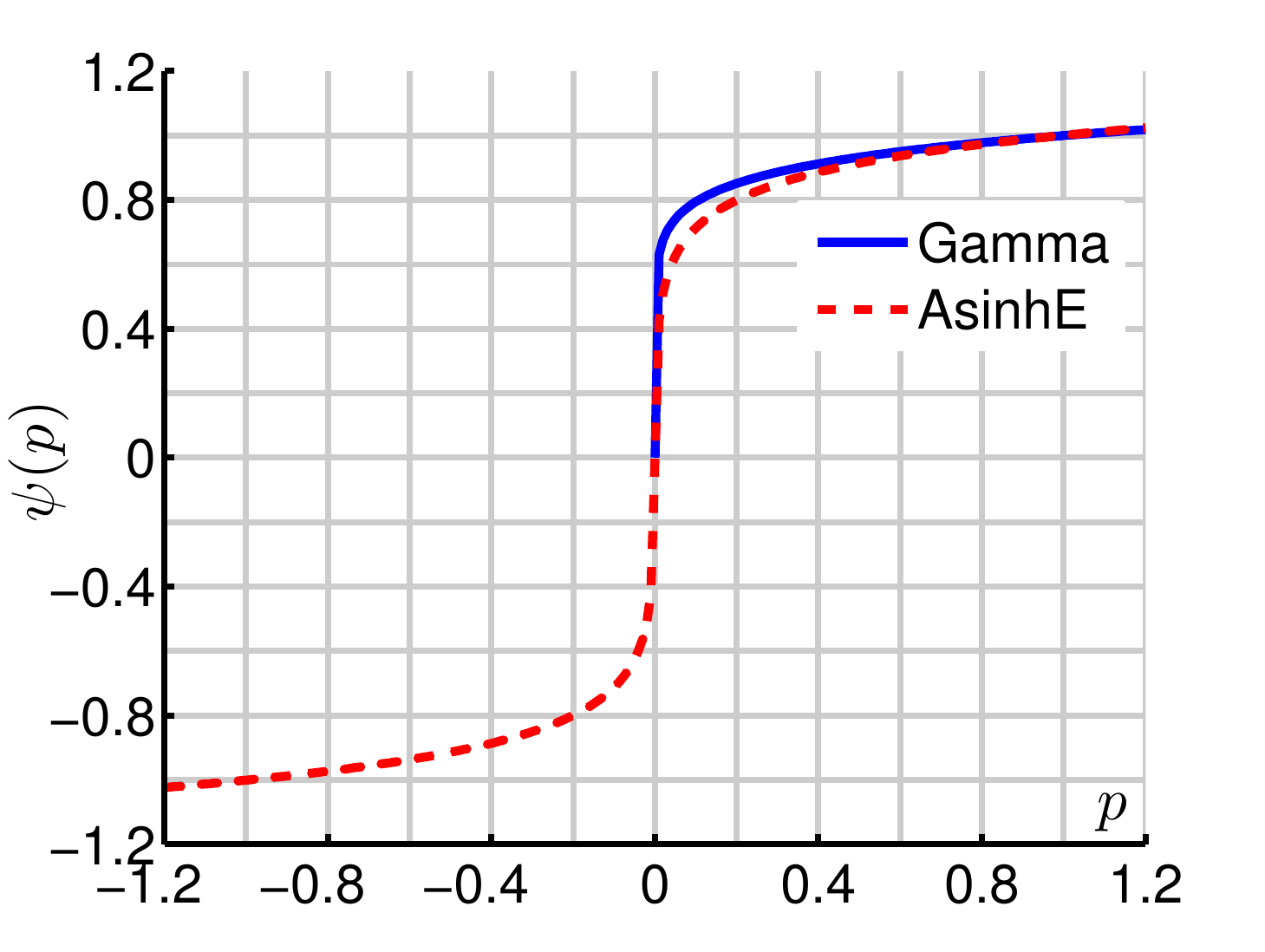}\vspace{-0.2cm}
\caption{\label{fig:pow1}}
\end{subfigure}
\begin{subfigure}[t]{0.24\linewidth}
\centering\includegraphics[trim=0 -18 0 15, clip=true, height=\PowHB, width=\PowW]{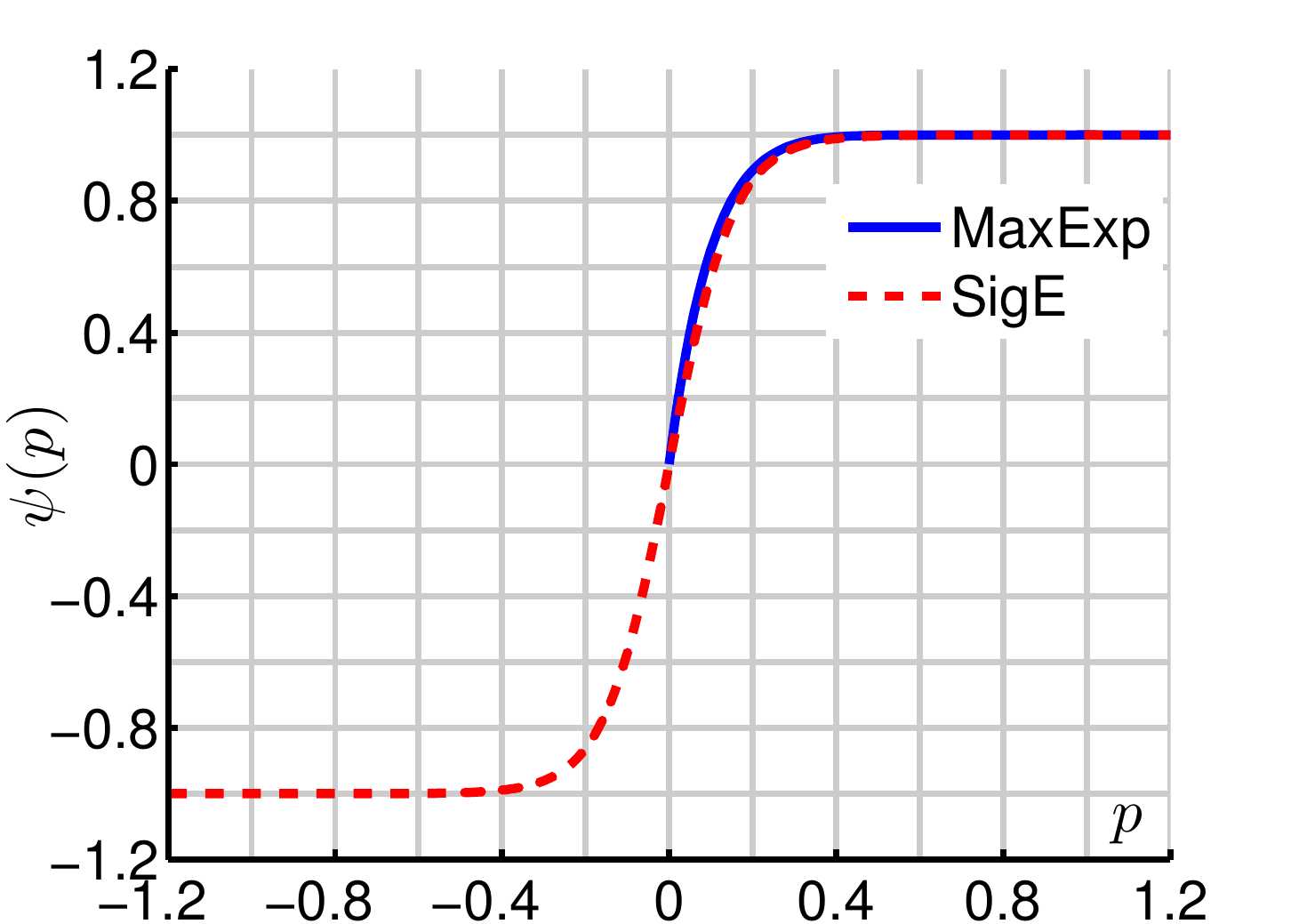}\vspace{-0.2cm}
\caption{\label{fig:pow2}}
\end{subfigure}
\begin{subfigure}[t]{0.24\linewidth}
\centering\includegraphics[trim=0 0 0 0, clip=true, height=\PowH, width=\PowW]{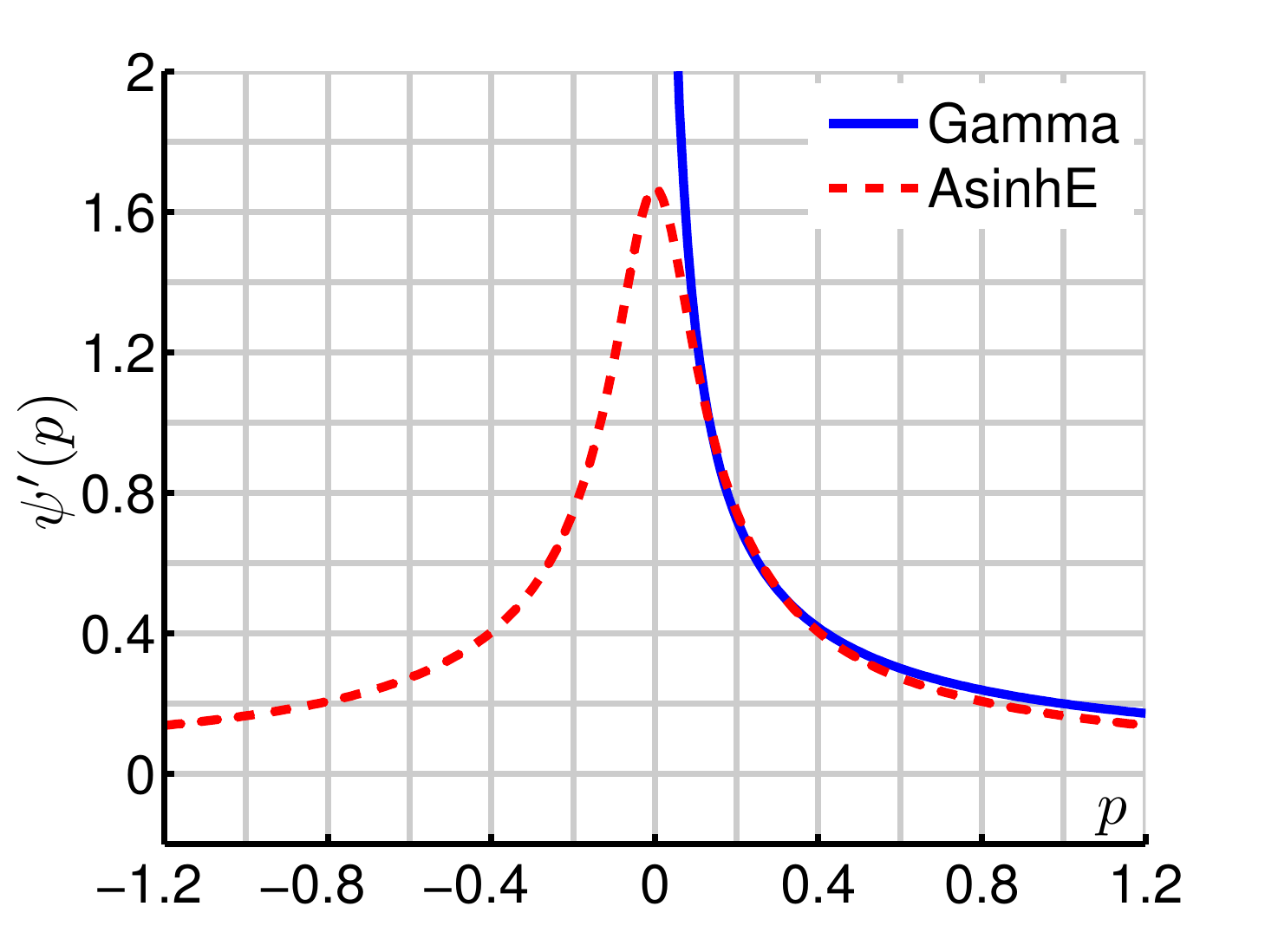}\vspace{-0.2cm}
\caption{\label{fig:pow3}}
\end{subfigure}
\begin{subfigure}[t]{0.24\linewidth}
\centering\includegraphics[trim=0 0 0 0, clip=true, height=\PowH, width=\PowW]{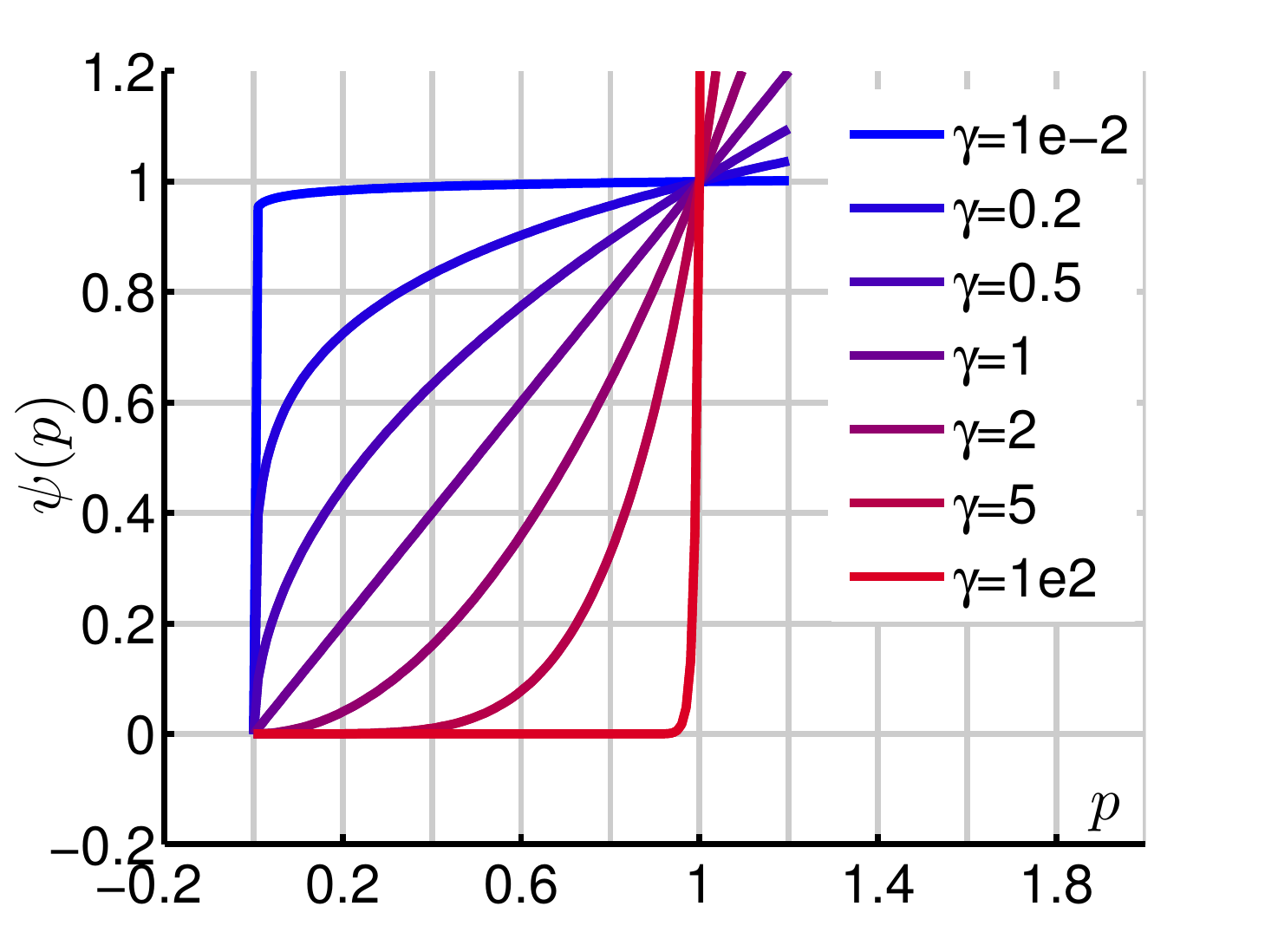}\vspace{-0.2cm}
\caption{\label{fig:pow4}}
\end{subfigure}
%
%\vspace{-0.7cm}
\caption{Gamma, AsinhE, MaxExp and SigmE are illustrated in Figures \ref{fig:pow1} and \ref{fig:pow2} while derivatives of Gamma and AsinhE are shown in Figure \ref{fig:pow3}. Lastly, Gamma for several values of $\gamma$ is shown in Figure \ref{fig:pow4} from which its similarity to MaxExp in range 0--1 is clear.}
\vspace{-0.3cm}
\label{fig:power-norms}
\end{figure*}

%\vspace{0.05cm}
%\noindent{\textbf{Co-occurrence matrix.}}
\subsection{Co-occurrence matrix}
\label{sec:cooc}

As in Prop. \ref{pr:linearize}, assume that datapoints $\mPhi_A\equiv\{\vphi_n\}_{n\in\tNnb_{\!A}}$ and $\mPhi_B\!\equiv\{\vphi^*_n\}_{n\in\tNnb_{\!B}}$ from two images $\Pi_A$ and $\Pi_B$ are given,  $N\!=\!|\tNnb_{\!A}|$ and $N^*\!\!=\!|\tNnb_{\!B}|$ are the numbers of data vectors obtained from the last convolutional feature map of CNN for images $\Pi_A$ and $\Pi_B$.
Moreover, assume that all $\vphi$ and $\vphi^*$ are rectified \eg, $\vphi_n\!:=\!\max(0,\vphi_n)$, $\vphi^*_n\!:=\!\max(0,\vphi^*_n)$, and subsequently $\beta$-centered w.r.t. the means $\vmu\!=\!\avg_{n\in\tNnb_A}\vphi_n$ and $\vmu^*\!\!=\!\avg_{n\in\tNnb_B}\vphi^*_n$ so that $\vphi_n\!:=\!\vphi_n\!-\!\beta\vmu$ and $\vphi^*_n\!:=\!\vphi^*_n\!-\!\beta\vmu^*\!$ for $0\!\leq\!\beta\!\leq\!1$.

The role of $\beta$-centering is to address anti-occurrences. Specifically, sophisticated models of Bag-of-Words utilize so-called negative visual words which are the evidence of lack of a given visual stimulus in an image. 
For instance, the authors of \cite{negoccur} define it as `{\em the negative evidence, i.e., a visual word that is mutually missing in two descriptions being compared}'. Lack of certain visual stimuli may correlate with certain visual classes \eg, lack of the sky may imply an indoor scene. Thus, the role of $\beta$ is to offset vectors $\vphi$ by their per-image averages $\vmu$ so that the positive/negative values yield correlations/anti-correlations, respectively. 

Next, let $x_n\!:=\!x_n/(W\!-\!1)$ and $y_n\!:=\!y_n/(H\!-\!1)$ be spatial coordinates normalized w.r.t. the width $W$ and height $H$ of conv. feature maps. % and $\vec{s}_n=[x_n; y_n]$. %  per patch $n$ and scale-related scalars $s_n\!:=\!s_n/(S\!-\!1)$. Moreover, we enumerate these coordinates as $0,\cdots,W\!-\!1$, etc.
We form the following kernel and its linearization by the use of Proposition \ref{pr:gaus_lin}:
\begin{align}
\label{eq:encode_sc}
&\!\!\!\!\!\left<\alpha\vvarphi(x_n,\vzeta), \alpha\vvarphi(x_{n'}^*\!,\vzeta)\right>\!+\!\left<\alpha\vvarphi(y_n,\vzeta), \alpha\vvarphi(y_{n'}^*\!,\vzeta)\right>\approx\nonumber\\
& \alpha^2G_{\sigma}(x_n\!-\!x_{n'}^*)\!+\!\alpha^2G_{\sigma}(y_n\!-\!y_{n'}^*).\!\!\! %+\alpha_s^2G_{\sigma_s}(s_n\!-\!s_n^*)
\end{align}
For $Z$ pivots $\vzeta$, we use $Z$ in range 3--10 and equally spaced intervals \eg, $\vzeta\!=\![-0.2:1.4/(Z\!-\!1):1.2]$ to encode the spatial coordinates $x_n$ and $y_n$. The above formulation extends to the aggregation over patches %processed one-by-one or 
extracted %at multiple scales 
from images as shown in Figure \ref{fig:principle}. 
We form vectors $\vphibar_n\!=\![\vphi_n; \vc_n]$ which are augmented by encoded spatial coordinates $\vc_n\!=\![\alpha\vvarphi(x_n,\vzeta); \alpha\vvarphi(y_n,\vzeta)]$. 
Thus, we define the total length of $\vc_n$ as $Z'\!\!=\!2Z$. %\alpha_s\vvarphi(s_n,\vzeta')
Combining the augmented vectors with the Proposition \ref{pr:linearize} and Eq. \eqref{eq:hok3} yields: % the following second-order model:
\begin{align}
%& \mPsi\left(\{\vphibar_n\}_{n\in\tNnb}\right)=\tG\Big(\frac{1}{N}\sum_{n\in\mathcal{N}}\!\vphibar_n\vphibar_n^T\Big).\label{eq:pn_simple1}
%& \mPsi\left(\{\vphibar_n\}_{n\in\tNnb}\right)=\left(\lambda+\frac{1}{N}\sum_{n\in\mathcal{N}}\!\vphibar_n\vphibar_n^T\right)^{\gamma}.\label{eq:pn_simple1}
%
& \mPsi\left(\{\vphibar_n\}_{n\in\tNnb}\right)=\tG(\mM)\,,\;\,\mM\!=\!\frac{1}{N}\sum\limits_{n\in\mathcal{N}}\vphibar_n\vphibar_n^T.\label{eq:pn_simple1}
\end{align}

%Figure \ref{fig:principle} shows the above operations in our pipeline.

\vspace{0.05cm}
\noindent{\textbf{Gamma pooling}} follows Remark \ref{re:pn} and is simply defined by setting $\tG(\mX)\!=\!(\lambda\!+\!\mX)^\gamma$, where rising $\mM$ to the power of $\gamma$ is element-wise and $\lambda$ is a small regularization constant:
\begin{align}
& \mPsi\left(\{\vphibar_n\}_{n\in\tNnb}\right)=\Big(\lambda+\frac{1}{N}\sum_{n\in\mathcal{N}}\!\vphibar_n\vphibar_n^T\Big)^{\gamma}.\label{eq:my_gamma1}
\end{align}

\subsection{Well-motivated Pooling Approaches}
Prop. \ref{prop:maxexp} states that quantity $1\!-\!(1\!-\!p)^N$ is the probability of at least one success being detected in the pool of the $N$ i.i.d. trials performed according to the Bernoulli distribution with the success probability $p$ and stored in $\vphi\!\in\!\{0,1\}^{N}$. Below we extend this simple theory to the case of co-occurrences.

\begin{proposition}
\label{pr:cooc}
Assume two event vectors $\vphi,\vphi'\!\!\in\!\{0,1\}^{N}$ which store the $N$ trials each, performed according to the Bernoulli distribution under i.i.d. assumption, 
 for which the probability $p$ of an event $(\phi_{n}\!\cap\!\phi'_{n}\!=\!1)$ denotes a co-occurrence and $1\!-\!p$, for $(\phi_{n}\!\cap\!\phi'_{n}\!=\!0)$, denotes the lack of it, and $p$ is estimated as an expected value $p\!=\!\avg_n\phi_n\phi'_{n}$. Then the probability of at least one co-occurrence event $(\phi_{n}\!\cap\!\phi'_{n}\!=\!1)$ in $\phi_n$ and $\phi'_n$ simultaneously in $N$ trials becomes:
\begin{equation}
\label{eq:my_maxexp3}
\psi\!=\!1\!-\!(1\!-\!p)^{N}.
%\vspace{-0.1cm}
\end{equation}
\end{proposition}
\begin{proof}
%We follow Proof \ref{pr:maxexp}. 
The probability of all $N$ outcomes to be $\{(\phi_{1}\!\cap\!\phi'_{1}\!=\!0),\cdots,(\phi_{N}\!\cap\!\phi'_{N}\!=\!0)\}$ amounts to $(1\!-\!p)^N$. 
The probability of at least one positive outcome 
$(\phi_{1}\!\cap\!\phi'_{1}\!=\!1)$ amounts to applying the logical `or' $\{(\phi_{1}\!\cap\!\phi'_{1}\!=\!1)\,|\cdots|\,(\phi_{N}\!\cap\!\phi'_{N}\!=\!1)\}$ and leads to $1\!-\!(1\!-\!p)^{N}\!$, where $p\!=\!\avg_n\phi_n\phi'_{n}$.

A stricter proof uses a Multinomial distribution model with four events for $(\phi_{n})$ and $(\phi'_{n})$ which describe all possible outcomes. Let probabilities $p, q, s$ and $1\!-\!p\!-\!q\!-\!s$ add up to 1 and correspond to events $(\phi_{n}\!\cap\!\phi'_{n}\!=\!1)$, $(\phi_{n}\!=\!1,\phi'_{n}\!=\!0)$, $(\phi_{n}\!=\!0,\phi'_{n}\!=\!1)$ and $(\phi_{n}\!\cup\!\phi'_{n}\!=\!0)$. The first event is a co-occurrence, the latter two are occurrences only and the last event is the lack of the first three events. The probability of at least one co-occurrence $(\phi_{n}\!\cap\!\phi'_{n}\!=\!1)$ in $N$ trials becomes:
\begin{align}
& \!\!\!\!\!\!\!\textstyle\sum\limits_{n=1}^{N}\sum\limits_{n'=0}^{N\!-n}\sum\limits_{n''\!=0}^{N\!-n-n'}\!\!\!\!\!\binom{N}{n,n'\!,n''\!,N\!-n-n'\!-n''\!} p^nq^{n'\!}s^{n''\!\!}(1\!\!-\!\!p\!\!-\!\!q\!\!-\!\!s)^{N\!-n-n'-n''}\!\!\!\!.\nonumber\\[-10pt]
& \label{eq:my_maxexppr2}
\end{align}
One can verify algebraically/numerically that Eq. \eqref{eq:my_maxexppr2} and \eqref{eq:my_maxexp3} are equivalent w.r.t. $p$ which completes the proof.
\end{proof}
\begin{remark}
\label{re:maxexp2}
A practical implementation of this pooling strategy is given by $\psi_{kl}\!=\!1\!-\!(1\!-\!\avg_n\phi_{kn}\phi_{ln})^{\eta}$, where $0\!<\!\eta\!\approx\!N$ is an adjustable parameter, and $\phi_{kn}$ and $\phi_{ln}$ are $k$-th and $l$-th features of an $n$-th feature vector \eg, as defined in Prop. \ref{pr:linearize}, which is normalized as detailed next.
\end{remark}
\begin{remark}
\label{re:maxexp3}
In practice, $p$ is an expected value over $N$ rectified co-occurring responses of pairs of convolutional filters rather than binary variables. A similar strategy is used with success in the BoW model \cite{me_tensor}. %These responses can be collected for instance over a spatial region of feature map associated with the $k^{\text{th}}$ convolutional filter or over co-occurring pairs of features:
In matrix form, we have:
\begin{align}
& \mPsi\!=\!\mygthree{\,\mM,\eta\,}\!=\!1\!-\!\left(1\!-\!\frac{\mM}{\trace(\mM)\!+\!\lambda}\right)^\eta, %\!\!,\;\;
%\mM\!=\!\frac{1}{N}\sum\limits_{n\in\mathcal{N}}\vphibar_n\vphibar_n^T,
\label{eq:my_maxexp4}
\end{align}
where $\trace(\mM)$ prevents elements of co-occurrence matrix $\mM$ in enumerator of Eq. \eqref{eq:my_maxexp4} from exceeding value of one, constant $\lambda$$\approx$1e-6 deals with the vanishing trace and $\eta$ is chosen via cross-validation. % and we have typically $\eta\!\approx\!N$.
\end{remark}
\begin{remark}
\label{re:maxexp_resid}
$\tG^{*}(\mM,\eta)\!=\!\tG(\mM,\eta)(\trace(\mM)\!+\!\lambda)^\gamma$ compensates for the trace in \eqref{eq:my_maxexp4} which affected the input-output ratio of norms. $\tG^{\ddagger}(\mM,\eta)\!=\!\tG(\mM,\eta)\!+\!\kappa\mM$ prevents vanishing gradients in pooling. Both terms can be combined.
\end{remark}

We note that matrix $\mM$ contains co-occurrences created from feature vectors $\vphi$ which were $\beta$-centered. Therefore, some entries of $\mM$ may be negative. This breaks down pooling models such as Gamma and MaxExp for which we strictly use $\beta\!=\!0$ that disables the anti-correlation mechanism. Nevertheless, we list detailed derivatives of these pooling functions w.r.t. the feature vectors in  Appendix \ref{app:der}.
%
%\end{remark}

\begin{table}[b]%[H]
\vspace{-0.3cm}
\begin{center}
{
%\setlength{\tabcolsep}{0.15em}
%\footnotesize
%\setlength{\tabcolsep}{0.15em}
\setlength{\tabcolsep}{0.3em}
\centering
%\hspace{-0.75cm}
\begin{tabular}{c | c c c c}
%\hline
Pooling  & $\psi(p)$ & $\psi'(p)$ & \multirow{2}{*}{$\psi(p)$} & \multirow{2}{*}{$\psi'(p)$} \\
function & if $p\!<\!0$ & if $p\!=\!0$ & &\\
\hline
{\em Gamma}$\;$ \cite{me_tensor} & \multicolumn{1}{c}{inv.} & $\infty$ & $p^{\gamma}$ & $\gamma p^{\gamma\!-\!1}$ \\
%
%\kern-0.9em 
{\em MaxExp} \cite{me_tensor} &  inv. & fin. & \kern-0.4em$1\!-\!(1\!-\!p)^{\eta}$ & \kern-0.3em$\eta(1\!-\!p)^{\eta\!-\!1}$ \\
\hline
{\em AsinhE} & ok & fin. & \kern-0.4em$\asinh(\gamma'p)$ & $\frac{\gamma'}{\sqrt{1+\gamma'^2p^2}}$\\
{\em SigmE} & ok & fin. & \kern-0.4em $\frac{2}{1\!+
\!e^{-\eta'p}}\!-\!1$ & $\frac{2\eta'\!e^{-\eta'p}}{(1+
e^{-\eta'p})^2}$
\end{tabular}
}
\end{center}
\vspace{-0.3cm}
\caption{A collection of Power Normalization functions.
Variables $\gamma\!>\!0$, $\gamma'\!\!>\!0$, $\eta\!\geq\!1$, and $\eta'\!\!\geq\!1$ control the level of power normalization. We indicate properties of $\psi$ such as finite ({\em fin.}) or infinite ($\infty$) derivative of $\psi$ w.r.t. $p$ at $p\!=\!0$ and invalid ({\em inv.}) or valid ({\em ok}) power normalization for $p\!<\!0$.}
\label{tab:smd}
\vspace{-0.1cm}
\end{table}

\subsection{Well-behaved Power Normalizations}
Power Normalizations in Eq. \eqref{eq:my_gamma1} and \eqref{eq:my_maxexp4} have infinite or undetermined gradients if coefficients $M_{mn}\!\rightarrow\!0$ and $\lambda\!\rightarrow\!0$. If regularization $\lambda\!>\!0$, both power normalizations are somewhat compromised as their role is to magnify weak signals $\phi\!\approx\!0$. Moreover, these pooling schemes break down in presence of negative entries $M_{mn}\!<\!0$. Therefore, we propose the following poolings extensions.

\vspace{0.05cm}
\noindent{\textbf{SigmE pooling}}, used in lieu of MaxExp in Eq. \eqref{eq:my_maxexp3} and \eqref{eq:my_maxexp4}, is given by Logistic a.k.a. Sigmoid ({\em SigmE}) functions:
\begin{align}
& \!\!\!\!\!\mPsi\!=\!\mygthree{\,\mM,\eta\,}\!=\!\frac{2}{1\!+\!\expl{-\eta'\mM}}\!-\!1\text{ and }\frac{2}{1\!+\!\expl{\frac{-\eta'\mM}{\trace(\mM)+\lambda}}}\!-\!1.\!%\quad\text{and}\quad\mM\!=\!\frac{1}{N}\sum\limits_{n\in\mathcal{N}}\vphibar_n\vphibar_n^T,
\end{align}

\vspace{0.05cm}
\noindent{\textbf{AsinhE pooling}} is an alternative to Gamma function in Eq. \ref{eq:my_gamma1}. It is defined as the Arcsin hyperbolic function:
\begin{align}
& \!\!\mPsi\!=\!\mygthree{\,\mM,\eta\,}\!=\arcsinh(\gamma'\!\mM)\!=\!\log(\gamma'\!\mM+\sqrt{1+{\gamma'}^2\!\mM^2}),
\end{align}

Figure \ref{fig:power-norms} illustrates MaxExp and SigmE as well as Gamma and AsinhE functions from which it is clear that,  for negative $p$, SigmE and AsinhE are natural extensions of MaxExp and Gamma, respectively. The derivative of AsinhE is smooth and finite (the same holds for SigmE) unlike the derivative of Gamma. Due to the above findings, we will perform our experiments on SigmE and AsinhE only. Table \ref{tab:smd}  lists various properties of the Power Normalization functions. Moreover, Appendix \ref{app:der2} provides detailed derivatives of these pooling functions w.r.t. the feature vectors. We used these derivatives in our end-to-end learning of CNNs.

Power Normalization functions have a whitening effect on features \ie, the frequent bursts of the same kind of feature are reduced while the responses of rarely occurring features are magnified \cite{me_tensor}. 
For co-occurrences of visual features, we showed in Prop. \ref{pr:cooc} that Power Normalizations act as detectors of co-occurring combinations of patterns \ie, they capture if at least one co-occurrence of features takes place but they discard the quantity of such co-occurrences which otherwise would be a source of nuisance/noise. % in classification.

\subsection{Spectral Power Normalizations}
Spectral versions of our pooling methods and their derivatives can be obtained by performing an SVD on $\mM$, substituting eigenvalues $\lambda_{ii}$ according to Table \ref{tab:smd} such that $\lambda^{\star}_{ii}\!:=\!\psi(\lambda_{ii})$ and computing $\mygthree{\mM}\!=\!\mU\mLambda^{\star}\mU^T$. For derivatives, $\lambda^{\diamond}_{ii}\!:=\!\psi'(\lambda_{ii})$ can be applied in back-propagation via SVD \cite{sminchisescu_matrix}. Table \ref{tab:smd2} shows that the spectral MaxExp and its derivative may be computed via matrix multiplications.

\begin{table}[b]%[H]
\vspace{-0.3cm}
%\renewcommand{\arraystretch}{0.9}
%\footnotesize
\setlength{\tabcolsep}{0.3em}
\ifdefined\arxiv
\centering
\else
\hspace{-0.5cm}
\fi
\begin{tabular}{c | c c c c}
%\hline
& {\em Gamma} & {\em MaxExp} & {\em AsinhE} & {\em SigmE}\\
\hline
\scriptsize$\!\!\!\mygthree{\mM}$\kern-0.2em & \kern-0.1em\scriptsize$\mM^\gamma$ & \kern-0.9em\scriptsize$\mIdent\!-\!(\mIdent\!-\!\!\frac{\mM}{{\trace(\mM)+\lambda}})^\eta$\kern-0.3em & \scriptsize$\log\!\Big(\!\gamma'\!\mM\!\!+\!(\mIdent\!\!+\!\!\gamma'^2\!\mM^2)^{\frac{1}{2}}\!\Big)$ & \kern-0.2em\scriptsize$2\Big(\mIdent\!\!+\!\!e^{\!\frac{-\eta'\mM}{{\trace(\mM)+\lambda}}}\Big)^{\!-\!1}\!\!\!\!\!-\!\mIdent$\kern-0.6em \\
der.\kern-0.2em & \kern-0.2em\scriptsize Eq. \eqref{eq:sylv} /SVD\kern-0.6em & \kern-0.1em\scriptsize Eq. \eqref{eq:binom_mat_der} /SVD & \scriptsize SVD & \scriptsize SVD
\end{tabular}
\caption{A collection of spectral Power Normalization functions. The square, square root, power, log and exp are matrix operations.}
\label{tab:smd2}
\vspace{-0.3cm}
\end{table}

\section{Experiments}
\label{sec:expts}

%In what follows, we detail datasets used in our experiments and 
Below we demonstrate experimentally merits of our second-order pooling with Power Normalizations.

\vspace{0.05cm}
\noindent{\textbf{Datasets.}} We employ four publicly available datasets and report the mean top-$1$ accuracy on each of them. The Flower102 dataset \cite{nilsback_flower102} is a fine-grained category recognition dataset that contains 102 categories of various flowers. % commonly occurring in the United Kingdom. 
Each class consists of between 40 and 258 images. 
The MIT67 dataset \cite{quattoni_mitindoors} contains a total of 15620 images belonging to 67 indoor scene classes. %For our experiments, we
We follow the standard evaluation protocol, which uses a train and test split of 80\% and 20\% of images per class.
The FMD dataset contains in total 100 images per category belonging to 10 categories of materials (\eg, glass, plastic, leather) collected from the Flickr website. %The authors ensured a variety of illumination conditions, compositions, colors, texture and material sub-types to challenge recognition algorithms.
Lastly, the Food-101 dataset \cite{food101} has 101000 images in total and 1000 images per category. %The training images are noisy and backgrounds vary.

\ifdefined\arxiv
\newcommand{\SrcImgWW}{0.13}
\newcommand{\SrcImgHHH}{1.8cm}
\newcommand{\SrcImgWWW}{1.8cm}
\else
\newcommand{\SrcImgWW}{0.215}
\newcommand{\SrcImgHHH}{1.8cm}
\newcommand{\SrcImgWWW}{1.8cm}
\fi

\begin{figure}[b]%htbp % left bottom right top
\centering%%%%
\vspace{-0.3cm}
\begin{subfigure}[b]{\SrcImgWW\linewidth}
\centering\includegraphics[trim=0 0 0 0, clip=true,width=\SrcImgWWW, height=\SrcImgHHH]{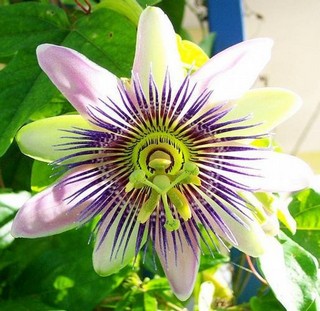}
\end{subfigure}
\begin{subfigure}[b]{\SrcImgWW\linewidth}
\centering\includegraphics[trim=0 0 0 0, clip=true,width=\SrcImgWWW, height=\SrcImgHHH]{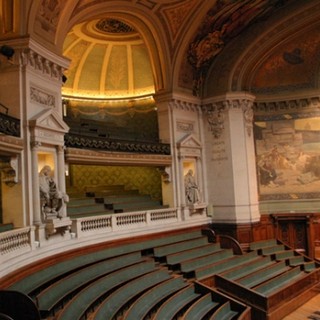}
\end{subfigure}
\begin{subfigure}[b]{\SrcImgWW\linewidth}
\centering\includegraphics[trim=0 0 0 0, clip=true,width=\SrcImgWWW, height=\SrcImgHHH]{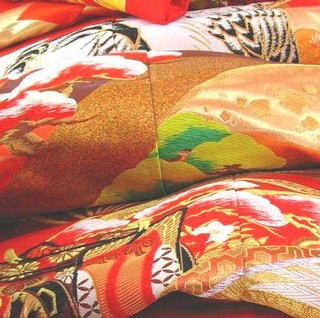}
\end{subfigure}
\begin{subfigure}[b]{\SrcImgWW\linewidth}
\centering\includegraphics[trim=0 0 0 0, clip=true,width=\SrcImgWWW, height=\SrcImgHHH]{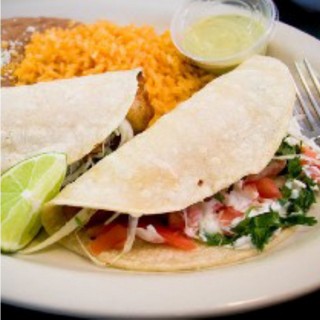}
\end{subfigure}
\\
\vspace{0.052cm}
\begin{subfigure}[b]{\SrcImgWW\linewidth}
\centering\includegraphics[trim=0 0 0 0, clip=true,width=\SrcImgWWW, height=\SrcImgHHH]{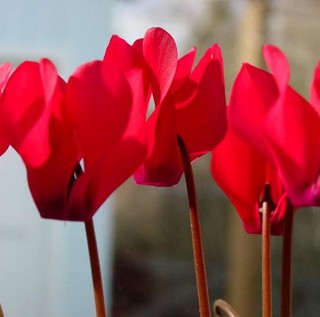}
\end{subfigure}
\begin{subfigure}[b]{\SrcImgWW\linewidth}
\centering\includegraphics[trim=0 0 0 0, clip=true,width=\SrcImgWWW, height=\SrcImgHHH]{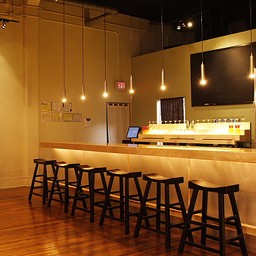}
\end{subfigure}
\begin{subfigure}[b]{\SrcImgWW\linewidth}
\centering\includegraphics[trim=0 0 0 0, clip=true,width=\SrcImgWWW, height=\SrcImgHHH]{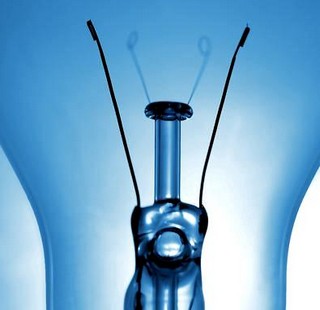}
\end{subfigure}
\begin{subfigure}[b]{\SrcImgWW\linewidth}
\centering\includegraphics[trim=0 0 0 0, clip=true,width=\SrcImgWWW, height=\SrcImgHHH]{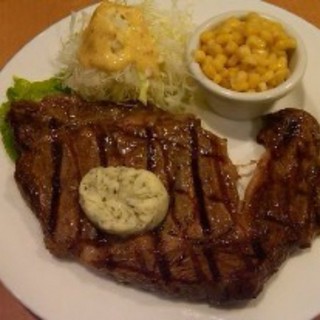}
\end{subfigure}
\\
\vspace{0.052cm}
\begin{subfigure}[b]{\SrcImgWW\linewidth}
\centering\includegraphics[trim=0 0 0 0, clip=true,width=\SrcImgWWW, height=\SrcImgHHH]{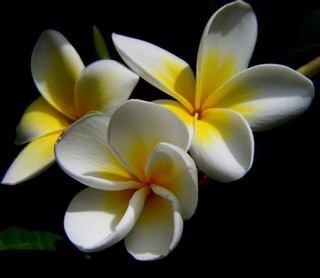}
\end{subfigure}
\begin{subfigure}[b]{\SrcImgWW\linewidth}
\centering\includegraphics[trim=0 0 0 0, clip=true,width=\SrcImgWWW, height=\SrcImgHHH]{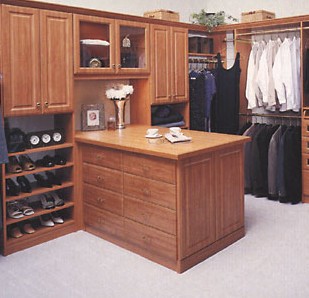}
\end{subfigure}
\begin{subfigure}[b]{\SrcImgWW\linewidth}
\centering\includegraphics[trim=0 0 0 0, clip=true,width=\SrcImgWWW, height=\SrcImgHHH]{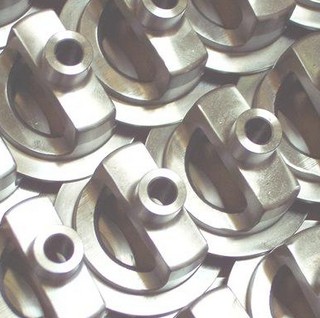}
\end{subfigure}
\begin{subfigure}[b]{\SrcImgWW\linewidth}
\centering\includegraphics[trim=0 0 0 0, clip=true,width=\SrcImgWWW, height=\SrcImgHHH]{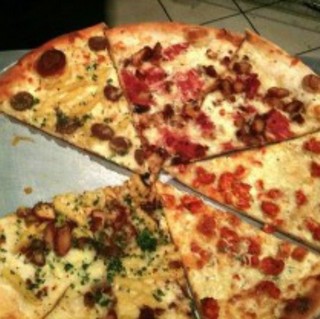}
\end{subfigure}
%
%
%\vspace{-0.7cm}
\caption{Each column shows examples of images from the Flower102, MIT67 FMD and Food101 dataset, respectively.
}\vspace{-0.3cm}
\label{fig:datasets}
\end{figure}

\vspace{0.05cm}
\noindent{\textbf{Experimental setup.}} 
For Flower102 \cite{nilsback_flower102}, we extract 12 cropped 224$\times$224 patches per image and use mini-batch of size 5 to fine-tune the ResNet-50 model \cite{resnet} pre-trained on ImageNet \cite{ILSVRC15}. We obtain 2048 dim. $12\!\times\!7\!\times\!7$ conv. feature vectors from the last conv. layer for our second-order pooling layer.
For MIT67 \cite{quattoni_mitindoors}, we resize original images to 336$\times$336 and use mini-batch of size 32, then fine-tune it on the ResNet-50 model \cite{resnet} pre-trained on the Places-205 dataset \cite{places_dataset}. With $336\!\times\!336$ image size, we obtain 2048 dim. $11\!\times\!11$ conv. feature vectors from the last conv. layer for our second-order pooling layer. 
For FMD \cite{fmd} and Food101 \cite{food101}, we resize images to $448\!\times\!448$, use mini-batch of size 32 and fine-tune ResNet-50 \cite{resnet} pre-trained on ImageNet \cite{ILSVRC15}. We use the 2048 dim. $14\!\times\!14$ conv. feature vectors from the last conv. layer. 
For ResNet-50, we fine-tune all layers for $\sim$20 epochs with learning rates 1e-4--1e-6. We use the Root Mean Square Propagation (RMSprop) \cite{rmsprop} with the moving average $0.99$. 
%We fine-tune it for 20 epochs and the learning rate is $10^(-4:20:-6)$.
%
 Where stated, we use AlexNet \cite{krizhevsky_alexnet} with fine-tuned last two conv. layers. We use $256$ dim. $6\!\times\!6$ conv. feature vectors from the last convolutional layer.

\vspace{0.05cm}
\noindent{\textbf{Our methods.}} We evaluate  the generalizations of MaxExp and Gamma %\cite{me_ATN,me_tensor} 
which are Logistic a.k.a. Sigmoid ({\em SigmE}) and the Arcsin hyperbolic ({\em AsinhE}) pooling functions. We focus mainly on our second-order representation ({\em SOP}) but we also occasionally report results for the first-order approach ({\em FOP}). For the baseline, we use the classifier on top of the {\em fc} layer ({\em Baseline}). The hyperparameters of our model are selected via cross-validation. The use of spatial coordinates is indicated by ({\em SC}) and spectral operators by ({\em Spec}).

\begin{table}[t]
\vspace{-0.3cm}
\centering
\begin{tabular}{l l|c|c|}
Method & & \multicolumn{2}{c|}{top-$1$ accuracy}  \\ 
\hline
{\em Second-order Bag-of-Words}\kern-0.6em  &\cite{me_tensor} & \multicolumn{2}{c|}{90.2} \\
{\em Factors of Transferability}\kern-0.6em &\cite{carlson_cnn} & \multicolumn{2}{c|}{91.3} \\
{\em Reversal-inv. Image Repr.}\kern-0.6em  &\cite{rrir} & \multicolumn{2}{c|}{94.0} \\
{\em Optimal two-stream fusion}\kern-0.6em  &\cite{two-stream} & \multicolumn{2}{c|}{94.5} \\
{\em Neural act. constellations}\kern-0.6em &\cite{nacc} & \multicolumn{2}{c|}{95.3} \\
\end{tabular}\\
%\caption{Flower102. State-of-the-art results from the literature.}
%\label{tab:flower102_soa}
%\end{table}
%
%\begin{table}
%\centering
\vspace{0.2cm}
\begin{tabular}{l l|c|c|}
Method       && Alexnet & ResNet-50 \\
\hline
{\em Baseline}      && 82.00  & 94.06   \\ 
{\em FOP}         	&& 85.40  & 94.08   \\ 
{\em FOP+AsinhE}    && 85.64  & 94.60   \\ 
{\em SOP}         	&& 87.20  & 94.70   \\ 
{\em SOP+AsinhE}    && 88.40  & 95.12   \\ 
{\em SOP+SC+AsinhE} && 90.70  & 95.74   \\ 
{\em SOP+SC+SigmE}  && 91.71  & \textbf{96.78}\\
\hline
{\em SOP+SC+Spec. Gamma} && -  & 96.88   \\ 
{\em SOP+SC+Spec. MaxExp}  && -  & \textbf{97.28}
\end{tabular}
\caption{The Flower102 dataset. The bottom part shows our results for Alexnet and ResNet-50. The top part of the table lists state-of-the-art results from the literature.}
\label{tab:flower102}
\vspace{-0.3cm}
\end{table}

\subsection{Evaluations}
\label{sec:eval}

\begin{table}[b]
\vspace{-0.3cm}
\centering
\begin{tabular}{l l|c|}
Method && top-$1$ accuracy\\
\hline
{\em CNNs with Deep Supervision} & \cite{places_mit_more} & 76.1\\
{\em Places-205}\kern-0.6em&\cite{places_mit} & 80.9 \\
{\em Deep Filter Banks}\kern-0.6em &\cite{cimpoi2015deep} & 81.0 \\
{\em Spectral Features}\kern-0.6em&\cite{khan2017scene} & 84.3 \\
\hline
%Baseline (Image Size=224) && 81.7\\
{\em Baseline}  && 84.0\\
%SOP-PN (Places-Resnet50) && 85.7 \\
{\em SOP+AsinhE} && 85.3 \\
{\em SOP+SigmE} && 85.6\\ 
{\em SOP+SC+AsinhE} && 85.9 \\
{\em SOP+SC+SigmE} && \textbf{86.3}
\end{tabular}
\caption{The MIT67 dataset. The bottom part shows our results for ResNet-50 pre-trained on the Places-205 dataset. The top part of the table lists state-of-the-art results from the literature. }
\label{tab:mit67}
\vspace{-0.3cm}
\end{table}

We start by combining first- and second-order representations with SigmE and AsinhE pooling. We also investigate the impact of AlexNet and ResNet-50 on our approach.

\vspace{0.05cm}
\noindent{\textbf{Flower102.}} Table \ref{tab:flower102} shows that AlexNet performs worse than ResNet-50 which is consistent with the literature. For the standard ResNet-50 fine-tuned on Flower102, we obtain 94.06\% accuracy. The first-order Average and AsinhE pooling ({\em FOP}) and ({\em FOP+AsinhE}) score 94.08 and 94.6\% accuracy. The second-order pooling ({\em SOP+AsinhE}) outperforms ({\em FOP+AsinhE}). We obtain the best result of \textbf{96.78}\% for the second-order representation combined with spatial coordinates and SigmE pooling ({\em SOP+SC+SigmE}) which is 2.72\% higher than our baseline. In contrast, a recent more complex state-of-the-art method  \cite{nacc} obtained 95.3\% accuracy. Our scores highlight that capturing co-occurrences of visual features and passing them via a well-defined Power Normalization function such as SigmE works well for our fine-grained problem. We attribute the good performance of SigmE to its ability to act as a detector of co-occurrences. The role of the Hyperbolic Tangent non-linearity popular in deep learning may be explained by its similarity to SigmE. Lastly, our spectral MaxExp ({\em SOP+SC+Spec. MaxExp}) yields \textbf{97.28}\% accuracy.

\noindent{\textbf{Scene recognition.}} Next, we validate our approach on MIT67--a larger dataset for scene recognition.  Table \ref{tab:mit67} shows that all second-order approaches ({\em SOP}) outperform the standard ResNet-50 network ({\em Baseline}) pre-trained on the Places-205 dataset and fine-tuned on MIT67. Moreover, ({\em SigmE})  yields marginally better results than ({\em AsinhE}). Using spatial coordinates ({\em SC}) also results in additional gain in the classification performance. The second-order representation combined with  spatial coordinates and SigmE pooling ({\em SOP+SC+SigmE}) yields \textbf{86.3}\% accuracy and outperforms our baseline and \cite{khan2017scene} by 2.3 and 2\%, respectively.

\vspace{0.05cm}
\noindent{\textbf{Material classification.}} Next, we quantify our performance on the FMD dataset for material/texture recognition. % which is different in its nature to the fine-grained and scene recognition problems. 
\begin{table}[t]
\vspace{-0.3cm}
\centering
\begin{tabular}{l l c || c c|}
Method  && acc. & Method  & acc.\\
\hline
{\em IFV+DeCAF}\kern-0.6em&\cite{cimpoi2014describing} & 65.5  & {\em Baseline} & 83.4 \\
{\em FV+FC+CNN}\kern-0.6em&\cite{cimpoi2015deep}       & 82.2  & {\em SOP+SC+AsinhE} & 85.0 \\
{\em SMO Task}\kern-0.6em&\cite{zhang2016integrating}  & 82.3  & {\em SOP+SC+SigmE} & \textbf{85.5}
\end{tabular}
\caption{The FMD dataset. Our (right) vs. other methods (left).}
\label{tab:fmd}
\vspace{-0.3cm}
\end{table}
Table \ref{tab:fmd} demonstrates that our second-order representation ({\em SOP+SC+SigmE}) scores \textbf{85.5}\% accuracy and outperforms our baseline approach by 2.1\%. We note that our approach and the baseline use the same testbed. The only difference is our second-order representations, spatial coordinates and Power Normalization components in our last layer.

\ifdefined\arxiv
\newcommand{\PlotWW}{0.245}
\newcommand{\PlotHHH}{4cm}
\newcommand{\PlotWWW}{3.8cm}
\else
\newcommand{\PlotWW}{0.495}
\newcommand{\PlotHHH}{4cm}
\newcommand{\PlotWWW}{4.45cm}
\fi

\ifdefined\arxiv
\begin{figure}[!b]%htbp % left bottom right top
\else
\begin{figure}[b]%htbp % left bottom right top
\fi
\vspace{-0.3cm}
\centering%%%%\vspace{-0.3cm}
\begin{subfigure}[t]{\PlotWW\linewidth}
\centering\includegraphics[trim=0 0 0 0, clip=true,width=\PlotWWW]{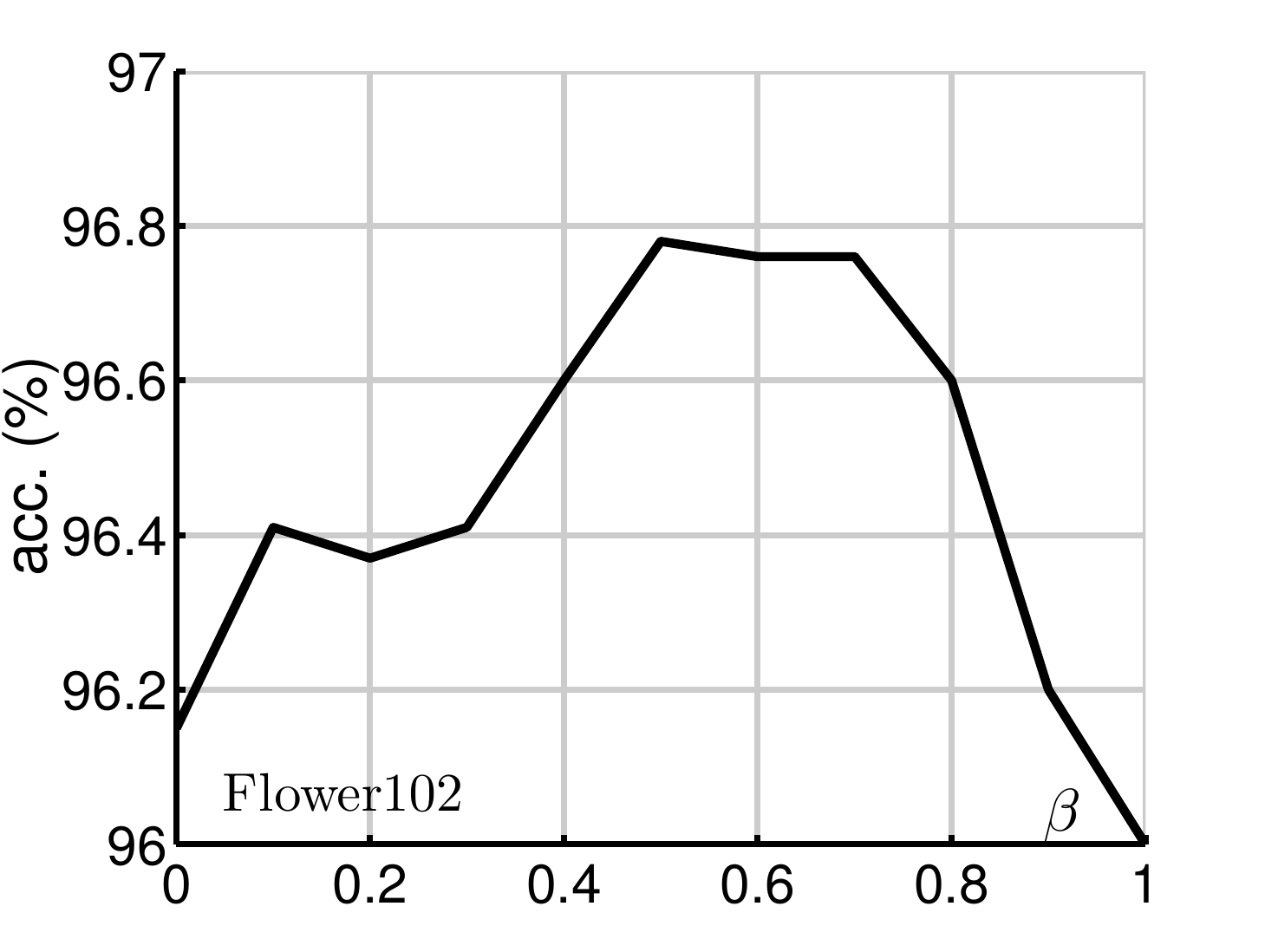}\vspace{-0.2cm}
\caption{\label{fig:eval1}}
\end{subfigure}
\begin{subfigure}[t]{\PlotWW\linewidth}
\centering\includegraphics[trim=0 0 0 0, clip=true,width=\PlotWWW]{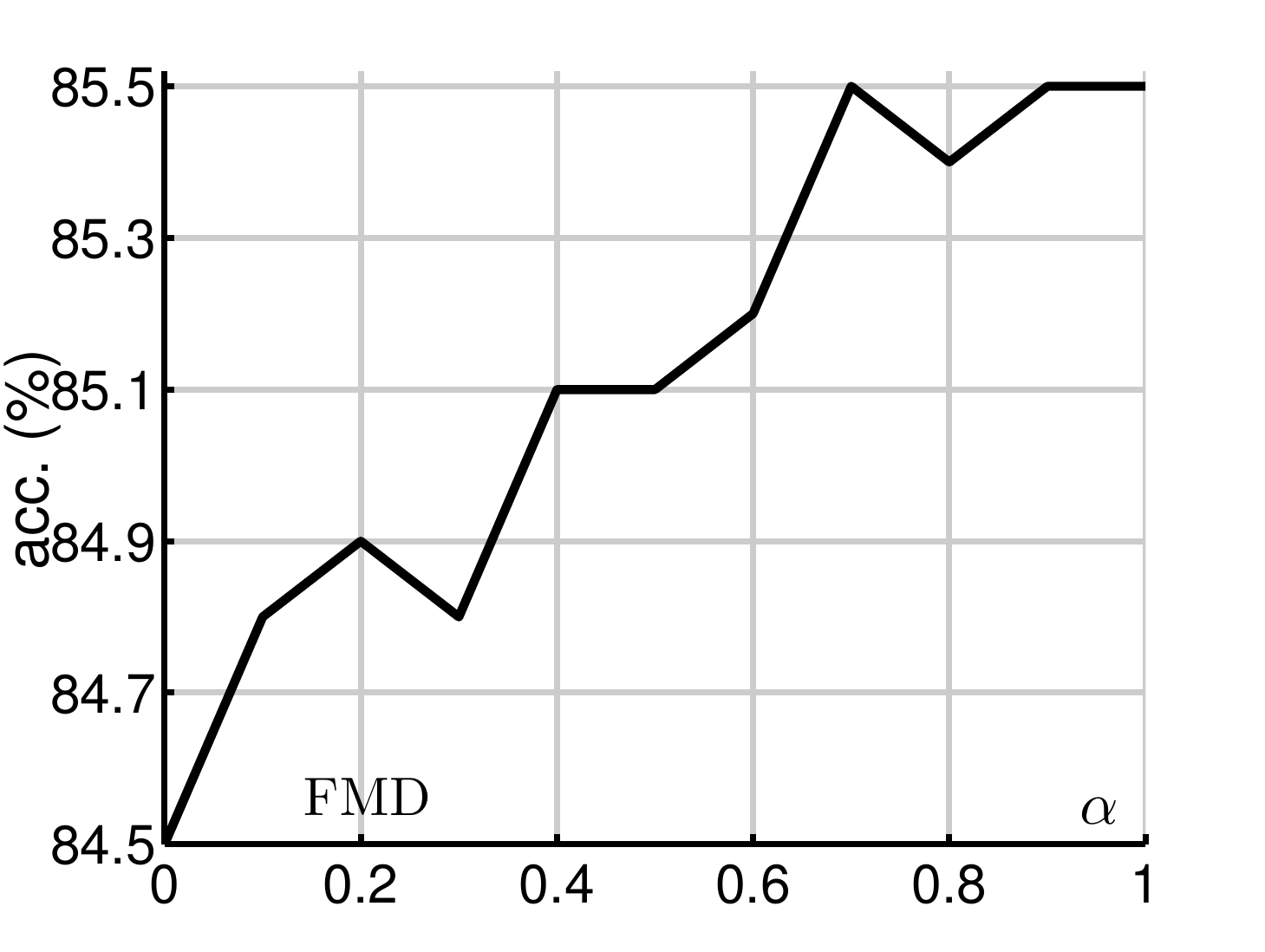}\vspace{-0.2cm}
\caption{\label{fig:eval2}}
\end{subfigure}
\ifdefined\arxiv\else\\\fi
\begin{subfigure}[t]{\PlotWW\linewidth}
\centering\includegraphics[trim=0 0 0 0, clip=true,width=\PlotWWW]{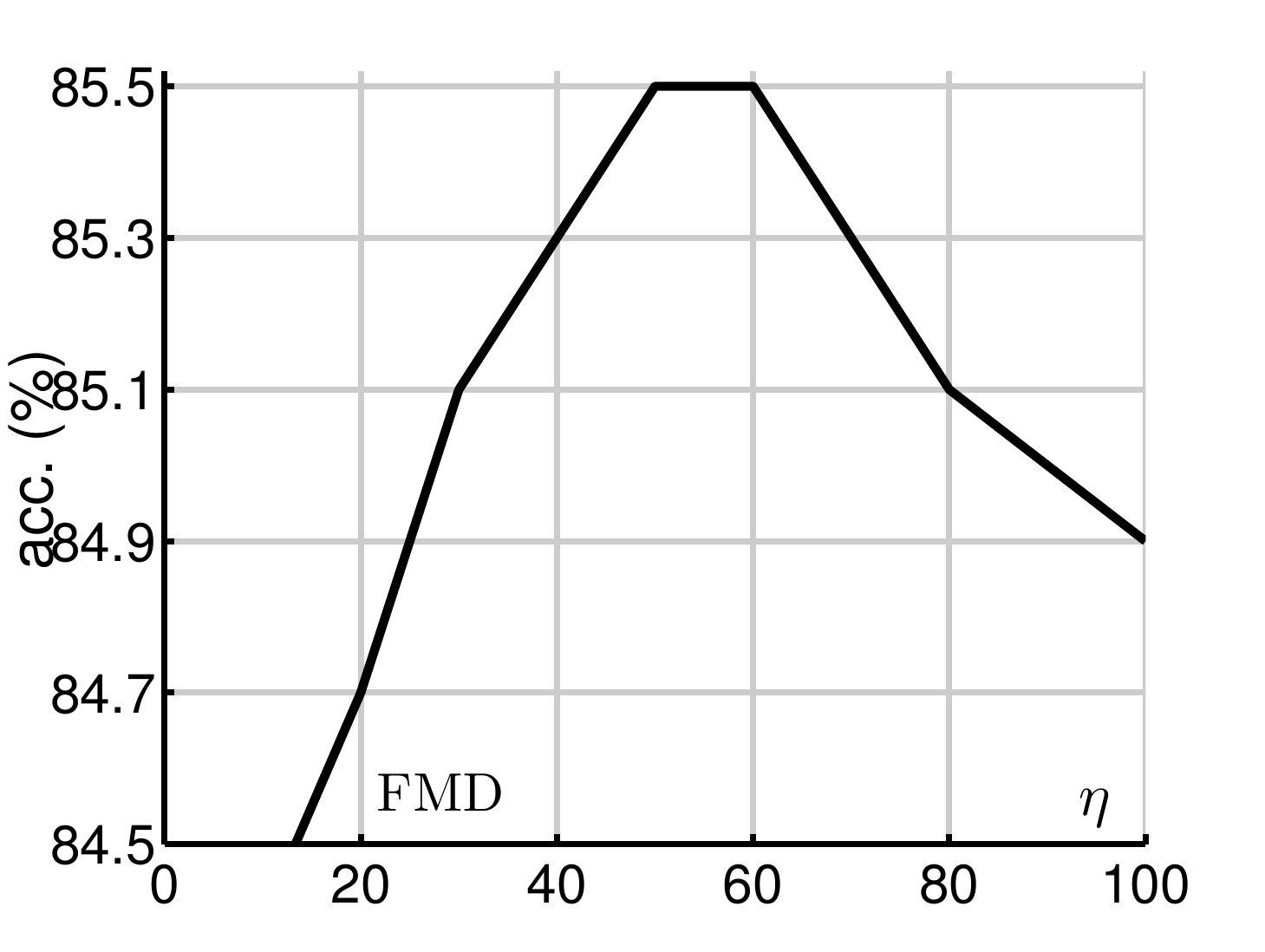}\vspace{-0.2cm}
\caption{\label{fig:eval3}}
\end{subfigure}
\begin{subfigure}[t]{\PlotWW\linewidth}
\centering\includegraphics[trim=0 0 0 0, clip=true,width=\PlotWWW]{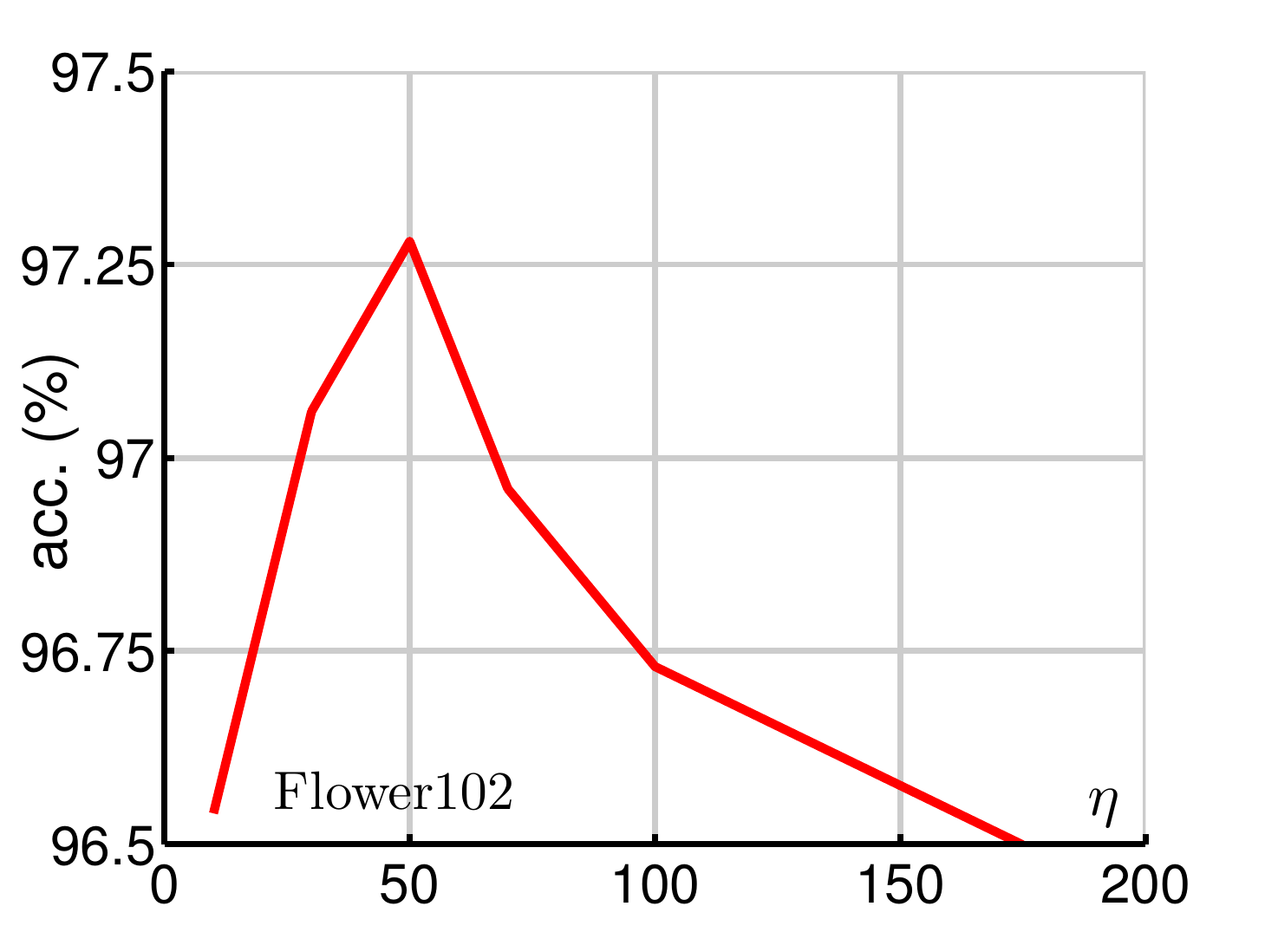}\vspace{-0.2cm}
\caption{\label{fig:eval4}}
\end{subfigure}
%
%
%\vspace{-0.7cm}
\caption{Performance w.r.t. hyperparameters. Figures \ref{fig:eval1} and \ref{fig:eval2}: $\beta$-centering on Flower102 and $\alpha$ for spatial coordinate encoding on FMD. Figures \ref{fig:eval3} and \ref{fig:eval4}: the accuracy w.r.t. the $\eta'\!$ and $\eta\!$ parameters given SigmE and the spectral MaxExp.
}
\label{fig:sens_pars}
\vspace{-0.3cm}
\end{figure}

\vspace{0.05cm}
\noindent{\textbf{Food101.}} We apply our strongest second-order representations ({\em SOP+SC+SigmE}) and ({\em SOP+SC+Spec. MaxExp}) to this dataset and obtain \textbf{87.5}\% and \textbf{87.8}\% accuracy. In contrast, a recent more involved kernel pooling \cite{cui2017kernel} reports 85.5\% accuracy while the baseline approach scores only 81.9\% in the same testbed. This demonstrates the strength of our approach on fine-grained problems.

\vspace{0.05cm}
\noindent{\textbf{Performance w.r.t. hyperparameters.}} Figure \ref{fig:eval1} demonstrates that $\beta$-centering has a positive impact on image classification with ResNet-50. This strategy, detailed in Section \ref{sec:cooc}, is trivial to combine with our pooling. Figure \ref{fig:eval2} shows that setting non-zero $\alpha$, which lets encode spatial coordinates according to Eq. \eqref{eq:encode_sc}, brings additional gain in accuracy at no extra cost. Figure \ref{fig:eval3} demonstrates that over 1\% accuracy can be gained by tuning our SigmE pooling. Moreover, Figure \ref{fig:eval4} shows that the spectral MaxExp can yield further gains over element-wise SigmE and MaxExp for carefully chosen $\eta$. Lastly, we have observed that our spectral and element-wise MaxExp converged in 3--12 and 15--25 iterations, resp. This shows that both spectral and element-wise pooling have their strong and weak points.

\section{Conclusions}
\label{sec:conclude}

We have studied Power Normalizations in the context of co-occurrence representations and demonstrated their theoretical role which is to `detect' co-occurring pairs of features. We have proposed surrogate functions SigmE and AsinhE which can handle so-called negative evidence and  have well-behaved derivatives for end-to-end learning. % which we performed.
SigmE and AsinhE also suggest that sigmoid-like non-linearities in neural networks reject counts of visual features and act as feature detectors instead. 
 Our pooling operators are element-wise and cheap to implement in GPU. Moreover, our pooling operators easily extend to spectral pooling. We have demonstrated state-of-the-art results on four popular benchmarks and sensible gains on powerful ResNet-50.
\vspace{-0.3cm}
\begin{appendices}

\section{Derivatives of Average, Gamma and MaxExp functions}
\label{app:der}
{\noindent Let} $\mPhi=[\vphi_1,\cdots,\vphi_N]\!\in\!\mbr{d\times N}$, $\mC=[\vc_1,\cdots,\vc_N]\!\in\!\mbr{Z'\!\times N}$, and some class. loss $\ell(\mPsi,\mW)$, where $\mPsi\!\in\!\semipd{d+Z'}\!$ (or $\spd{}$) and $\mW$ are our descriptor and a hyperplane. 
Eq. \eqref{eq:pn_simple1} yields:
\begin{align}
& \frac{\partial\sum_n\!\vphibar_n\vphibar_n^T}{\partial \phi_{m'n'}}\!=\!
\left[\begin{array}{cc}
\vj_{m'}\vphi_{n'}^T\!+\!\vphi_{n'}\vj_{m'}^T & \vj_{m'}\vc_{n'}^T \\
\vc_{n'}\vj_{m'}^T & [0]_{Z'\!\times Z'\!}  \\
\end{array}\right],
\label{eq:der_auto}
\end{align}
where $[0]_{Z'\!\times Z'\!}$ denotes array of size $Z'\!\times Z'\!$ filled with zeros. 

\vspace{0.05cm}
{\noindent\textbf{Average pooling}} is set by $\tG(\mM)\!=\!\mM$ and $\mD\!=\!\vOnes\vOnes^T$ so that $\mPsi\!=\!\mM\!=\!\frac{1}{N}\sum_n\!\vphibar_n\vphibar_n^T$. Thus, the full derivative becomes:
\begin{align}
& \!\!\!\text{\scriptsize $\sum\limits_{k,l}\frac{\partial \ell(\mPsi,\mW)}{\partial  \Psi_{kl}}\frac{\partial \Psi_{kl}}{\partial \mPhi}=\frac{2}{N}\sym\Big(\frac{\partial \ell(\mPsi,\mW)}{\partial  \mPsi}\!\odot\!\mD\Big)_{(1:d,:)}
\left[\!\!\begin{array}{c}
\mPhi\\
\mC
\end{array}\!\!\right]$}.
\label{eq:der_auto2}
\end{align}

\vspace{0.05cm}
{\noindent\textbf{Gamma pooling}} is set by $\mPsi\!=\!\tG(\mM)\!=\!(\lambda\!+\!\mM)^\gamma$, where rising $\mM$ to the power of $\gamma$ is element-wise and $\lambda$ is a reg. constant. Thus, we obtain:% we obtain $\mPsi=\left(\lambda+\frac{1}{N}\sum_n\!\vphibar_n\vphibar_n^T\right)^\gamma$ which yields:
\begin{align}
&\!\!\!\frac{\partial\mPsi}{\partial\phi_{m'n'}} =\frac{1}{N}\gamma\big(\lambda\!+\!\mM\big)^{\gamma-1}\!\odot\frac{\partial\sum_n\!\vphibar_n\vphibar_n^T}{\partial \phi_{m'n'}}.\!\!
\end{align}
The derivative is given by Eq. \eqref{eq:der_auto2} if $\mD\!=\!\gamma\big(\lambda\!+\!\mM\big)^{\!\gamma-1}$. %the following chain rule:
\comment{
\begin{align}
&\!\!\!\!\!\!\sum\limits_{k,l}\frac{\partial \ell(\mPsi,\mW)}{\partial  \Psi_{kl}}\frac{\partial \Psi_{kl}}{\partial \mPhi}=\\
&\quad\frac{2\gamma}{N}\sym\left(\frac{\partial \ell(\mPsi,\mW)}{\partial  \mPsi}\!\odot\!\big(\lambda\!+\!\mM\big)^{\!\gamma-1}\right)_{(1:d,:)}\!
\left[\!\!\begin{array}{c}
\mPhi\\
\mC
\end{array}\!\!\right],\nonumber
\end{align}
where $\mM_{(1:d,1:n)}$ denotes a MATLAB style operator selecting sub-matrix $\mM'\!\in\!\mbr{d\times n}$ from $\mM$ such that $\mM'_{d'n'}\!=\!\mM_{d'n'}, \forall d'\!\!=\!1,\cdots,d,\;n'\!\!=\!1,\cdots,n$.
}

\vspace{0.05cm}
{\noindent\textbf{MaxExp pooling}} $\mPsi\!=\!\tG(\mM)\!=\!1\!-\!(1\!-\!\mM/(\trace(\mM)+\lambda))^\eta$ has the derivative given by  Eq. \eqref{eq:der_auto2} with the following $\mD$:
%The derivatives w.r.t. variables of this model are given below:
%
\comment{
\begin{align}
& \!\!\!\!\frac{\partial\mPsi}{\partial \eta}\!=\!-\!\left(1\!-\!\frac{\mM}{\trace(\mM)+\lambda}\right)^\eta\!\!\odot\left(\log\left(1\!-\!\frac{\mM}{\trace(\mM)+\lambda}\right)\right)^{-1}\!\!\!\!,\\
& \!\!\!\!\left[\frac{\partial\Psi_{m'n'}}{\partial M_{m'n'}}\right]_{
\begin{array}{c}
\!\!\!\scriptscriptstyle(m'\!,n')\in\!\!\!\!\\[-5pt]
\!\!\scriptscriptstyle\idx{d}\times\idx{d}\!\!\!\!
\end{array}
}\!\!\!\!\!\!\!\!=\eta\left(1\!-\!\frac{\mM}{\trace(\mM)+\lambda}\right)^{\eta-1}\nonumber\\
&\qquad\qquad\qquad\odot\left(\frac{1}{\trace(\mM)\!+\!\lambda}\!-\!\frac{\mM\!\odot\!\mIdent}{\left(\trace(\mM)\!+\!\lambda\right)^2}\right),
\end{align}
while the matrix form of this derivative becomes:
}
\comment{
\begin{align}
& \!\!\!\!\frac{\partial\mPsi}{\partial\phi_{m'n'}}\!=\!-\frac{\eta}{N}\left(1\!-\!\frac{\mM}{\trace(\mM)+\lambda}\right)^{\eta-1}\\
&\odot\left(\frac{1}{\trace(\mM)\!+\!\lambda}\!-\!\frac{\mM\!\odot\!\mIdent}{\left(\trace(\mM)\!+\!\lambda\right)^2}\right)\!\odot\!\,\frac{\partial\sum_n\!\vphibar_n\vphibar_n^T}{\partial \phi_{m'n'}},\nonumber
\end{align}
}
\begin{align}
& \!\!\!\!\text{\scriptsize $\mD\!=\!\eta\left(1\!-\!\frac{\mM}{\trace(\mM)+\lambda}\right)^{\eta-1}\!\!\!\!\!\!\!\odot\mT\text{ and }\; \mT\!=\!\left(\frac{1}{\trace(\mM)\!+\!\lambda}\!-\!\frac{\mM\!\odot\!\mIdent}{\left(\trace(\mM)\!+\!\lambda\right)^2}\right)$},
\end{align}
%
%which can be further rewritten as follows:
%
\comment{
\begin{align}
&\!\!\!\!\!\!\!\!\sum\limits_{k,l}\frac{\partial \ell(\mPsi,\mW)}{\partial  \Psi_{kl}}\frac{\partial \Psi_{kl}}{\partial \mPhi}\!=\!-\frac{2\eta}{N}\sym\!\left(\frac{\partial \ell(\mPsi,\mW)}{\partial  \mPsi}\!\odot\! \Big(\!1\!-\!\frac{\mM}{\trace(\mM)\!+\!\lambda}\!\Big)^{\eta-1} \right.\nonumber\\
&\quad\left.\odot\Big(\frac{1}{\trace(\mM)\!+\!\lambda}\!-\!\frac{\mM\!\odot\!\mIdent}{\left(\trace(\mM)\!+\!\lambda\right)^2}\Big)\right)_{(1:d,:)}\!
\left[\!\!\begin{array}{c}
\mPhi\\
\mC
\end{array}\!\!\right],
\end{align}
}
where multiplication $\odot$, division, rising to the power {\em etc.} are all element-wise operations.

\section{Derivatives of SigmE and AsinhE pooling}
\label{app:der2}

\vspace{0.05cm}
{\noindent\textbf{SigmE pooling}} is set by $\mPsi\!=\!\tG(\mM)\!=\!\frac{2}{1\!+\!\expl{-\eta'\mM}}\!-\!1$ or trace-normalized $\frac{2}{1\!+\!\expl{\frac{-\eta'\mM}{\trace(\mM)+\lambda}}}\!-\!1$. The first expression yields: % the derivative:
\begin{align}
& \!\!\!\!\frac{\partial\mPsi}{\partial\phi_{m'n'}}\!=\!\frac{1}{N}\frac{2\eta'\expl{-\eta'\mM}}{(1+\expl{-\eta'\mM})^2}\odot(\vj_{m'}\vphi_{n'}^T\!+\!\vphi_{n'}\vj_{m'}^T),
\end{align}
where multiplication $\odot$, division, and exponentiation are all element-wise operations.

\vspace{0.05cm}
{\noindent\textbf{AsinhE pooling}} is set by $\mPsi\!=\!\tG(\mM)\!=\!\arcsinh(\gamma'\!\mM)\!=\!\log(\gamma'\!\mM+\sqrt{1+{\gamma'}^2\!\mM^2})$ which yields the following:
\begin{align}
& \!\!\!\!\!\frac{\partial\mPsi}{\partial\phi_{m'n'}}\!=\!\frac{1}{N}\frac{\gamma'}{\sqrt{{\gamma'}^2\mM^2+1}}\odot(\vj_{m'}\vphi_{n'}^T\!+\!\vphi_{n'}\vj_{m'}^T),
\end{align}
where multiplication $\odot$, division, square root and the square are all element-wise operations.

For SigmE, trace-normalized SigmE and AsinhE pooling methods, %we enable them if we set $\mPsi\!=\!\tG(\mM)\!=\!\frac{2}{1+\expl{-\eta'\mM}}-1$ or $\mPsi\!=\!\tG(\mM)\!=\!\log(\gamma'\!\mM+\sqrt{1+{\gamma'}^2\!\mM^2})$, respectively.
\comment{
Then, for $\mPhi=[\vphi_1,\cdots,\vphi_N]\!\in\!\mbr{d\times N}$ and $\mC=[\vc_1,\cdots,\vc_N]\!\in\!\mbr{Z'\!\times N}$, we obtain the following expression:
\begin{align}
& \!\!\sum\limits_{k,l}\frac{\partial \ell(\mPsi,\mW)}{\partial  \Psi_{kl}}\frac{\partial \Psi_{kl}}{\partial \mPhi}=\frac{2}{N}\sym\left(\frac{\partial \ell(\mPsi,\mW)}{\partial  \mPsi}\!\odot\! 
\mD
\right)_{(1:d,:)}
\left[\!\!\begin{array}{c}
\mPhi\\
\mC
\end{array}\!\!\right],
\label{eq:general_der_add_pn}
\end{align}
}
the final derivatives are given by Eq. \eqref{eq:der_auto2} with the following $\mD$, respectively: % is used (respectively):
\begin{align}
&\!\!\!\!\!\!\text{\scriptsize $\mD\!=\!\frac{2\eta'\expl{-\eta'\mM}}{(1\!+\!\expl{-\eta'\mM})^2}$}\text{ or } \text{\scriptsize $\mD\!=\!\frac{2\eta'\expl{\frac{-\eta'\mM}{\trace(\mM)+\lambda}}}{\big(1\!+\!\expl{\frac{-\eta'\mM}{\trace(\mM)+\lambda}}\big)^2}\!\odot\mT$} \text{ and } \text{\scriptsize $\mD\!=\!\frac{\gamma'}{\sqrt{{\gamma'}^2\mM^2\!+\!1}}$}.
\label{eq:additional_pn_matrices_d}
\end{align}
Moreover, for  SigmE and AsinhE we allow $\beta$-centering so that $\vphi_n\!:=\!\vphi_n\!-\!\beta\vmu$ and $\vphi^*_n\!:=\!\vphi^*_n\!-\!\beta\vmu^*\!$. Thus, the derivative of this substitution has to be included in the chain rule.

\section{Derivatives of Spectral Gamma and MaxExp$\!\!$}
\label{app:der3}
%Back-propagation on the matrix-form of Gamma and MaxExp can be obtained via back-propagation on SVD. However, there exist alternative formulations.

\vspace{0.05cm}
{\noindent\textbf{Gamma pooling}} has derivative which can be solved by the SVD back-propagation or the Sylvester equation if $\gamma\!=0.5$:
\begin{align}
&\!\!\text{\scriptsize $2\res\!\Big(\sym\Big(\frac{\partial \ell(\mPsi,\mW)}{\partial  \mPsi}\Big)^T_{(:)}\!\!\mM^*\!\Big)_{d\!+\!Z'\!\times d\!+\!Z'\!} \text{ and } \mM^*\!\!\!\!=\!(\mIdent\!\otimes\!\mM^{\frac{1}{2}}\!\!+\!\!\mM^{\frac{1}{2}}\!\otimes\!\mIdent)^{\dagger}$},
\label{eq:sylv}
\end{align}
where $\otimes$ and $\dagger$ are the Kronecker product and the pseudo-inverse. Matrix vectorization and reshaping to the size $m\!\times\!n$ are denoted by $(:)$ and $\res(\mX)_{m\!\times\!n}$.

\vspace{0.05cm}
{\noindent\textbf{MaxExp}} has a closed-form derivative which requires the following chain rule:
\begin{align}
& \!\!\!\!\!\text{\scriptsize $\frac{\partial\mygthree{\mM}}{\partial\ M_{kl}}=\!\frac{1}{\trace(\mM)}\sum\limits_{n\!=\!0}^{{\eta}-1}\left(\mIdent\!-\!\frac{\mM}{\trace(\mM)}\right)^n\!\left(\mJ_{kl}-\frac{\mM}{\trace(\mM)}\sIdent_{kl}\right)\left(\mIdent\!-\!\frac{\mM}{\trace(\mM)}\right)^{{\eta}-1-n}\!\!\!\!\!\!\!\!\!\!\!\!\!$}.
\label{eq:binom_mat_der}
\end{align}

\end{appendices}

%\begin{appendices}
%\renewcommand\title[1]{}
%\newcommand\titlerunning[1]{}
%\newcommand\authorrunning[1]{}
%\renewcommand\author[1]{}
%\newcommand\institute[1]{}
%\newcommand\maketitle{}
%\input{supp}
%\end{appendices}

{\small
%\bibliographystyle{plainnat}
%\bibliography{pn}

}

\end{document}